\documentclass{article}

    \PassOptionsToPackage{numbers, compress}{natbib}

    \usepackage[preprint]{neurips_2021}

\usepackage[utf8]{inputenc} 
\usepackage[T1]{fontenc}    
\usepackage{hyperref}
\hypersetup{
    colorlinks=true,
    linkcolor=black,
    citecolor=blue
}
\usepackage{times}
\usepackage{url}
\usepackage{amsfonts}       
\usepackage{nicefrac}       
\usepackage{microtype}      
\usepackage{amsthm}
\usepackage{amsmath,mathtools}
\usepackage{empheq}
\usepackage{amssymb}
\usepackage{macros}
\usepackage{booktabs}
\usepackage{multirow}
\usepackage{subcaption}
\usepackage{wrapfig}
\usepackage{chngcntr}
\usepackage{xspace}
\usepackage{bm}
\usepackage[ruled,vlined,linesnumbered]{algorithm2e}
\usepackage{nicefrac,xfrac}
\usepackage{multirow}
\usepackage[table,dvipsnames]{xcolor}
\usepackage{tcolorbox}
\usepackage{makecell}
\usepackage{algorithmic}
\usepackage{pifont}
\newcommand{\cmark}{{\color{PineGreen}\ding{51}}}%
\newcommand{\xmark}{{\color{BrickRed}\ding{55}}}%

\colorlet{mylinkcolor}{violet}
\colorlet{mycitecolor}{YellowOrange}
\colorlet{myurlcolor}{Aquamarine}

\newcommand{\vc}{\bm{c}}
\newcommand{\vx}{\bm{x}}
\newcommand{\vy}{\bm{y}}

\newcommand{\ve}{\bm{e}}
\newcommand{\vs}{\bm{s}}
\newcommand{\matA}{\bm{A}}
\newcommand{\matB}{\bm{B}}
\newcommand{\matH}{\bm{H}}
\newcommand{\matI}{\bm{I}}
\newcommand{\matP}{\bm{P}}

\newcommand{\matK}{\bm{K}}

\newcommand{\matM}{\bm{M}}
\newcommand{\matN}{\bm{N}}

\newcommand{\opA}{\mathcal{A}}

\newcommand{\obj}{F}
\newcommand{\numClients}{\ensuremath{M}}

\newcommand{\localStep}{\tau}
\newcommand{\data}{\ensuremath{\mathcal{D}}}
\newcommand{\clientDist}{\ensuremath{\mathcal{C}}}
\newcommand{\clientWeight}{w}
\newcommand{\activeClients}{\mathcal{S}}
\newcommand{\sgrad}{g}

\newcommand{\localChange}{\Delta}

\newcommand{\lr}{\eta}
\newcommand{\slr}{\alpha}
\newcommand{\ntg}{\bm{h}}

\newcommand{\lip}{L}

\newcommand{\ie}{\textit{i.e.,}\xspace}
\newcommand{\eg}{\textit{e.g.,}\xspace}

\newcommand{\fedavg}{\textsc{FedAvg}\xspace}

\newcommand{\fedopt}{\textsc{FedOpt}\xspace}

\newcommand{\adagrad}{\textsc{AdaGrad}\xspace}
\newcommand{\adam}{\textsc{Adam}\xspace}

\usepackage[colorinlistoftodos,prependcaption,disable]{todonotes}
\newcommand{\JW}[1]{\todo[color=yellow!25, inline]{ Jianyu: #1} \index{Jianyu: !#1}}
\newcommand{\zheng}[1]{\todo[color=blue!25, inline]{ Zheng: #1} \index{Zheng: !#1}}

\newcommand{\newadd}[1]{{\color{black}#1}}

\usepackage{enumitem}
\usepackage{cleveref}
\crefname{equation}{}{}
\Crefname{equation}{}{}
\crefname{thm}{theorem}{theorems}
\Crefname{thm}{Theorem}{Theorems}
\crefname{clm}{claim}{claims}
\Crefname{clm}{Claim}{Claims}
\Crefname{coro}{Corollary}{Corollaries}
\Crefname{lem}{Lemma}{Lemmas}
\Crefname{sec}{Section}{Sections}
\crefname{app}{appendix}{appendices}
\Crefname{app}{Appendix}{Appendices}
\Crefname{part}{Part}{Parts}
\crefname{prop}{proposition}{propositions}
\Crefname{prop}{Proposition}{Propositions}
\Crefname{propty}{Property}{Properties}
\crefname{figure}{fig.}{figures}
\Crefname{figure}{Figure}{Figures}
\crefname{defn}{definition}{definitions}
\Crefname{defn}{Definition}{Definitions}
\crefname{fact}{fact}{facts}
\Crefname{fact}{Fact}{Facts}
\crefname{appendix}{appendix}{appendices}
\Crefname{appendix}{Appendix}{Appendices}
\crefname{algo}{algorithm}{algorithms}
\Crefname{algo}{Algorithm}{Algorithms}
\crefname{algorithm}{algorithm}{algorithms}
\Crefname{algorithm}{Algorithm}{Algorithms}
\crefname{conj}{conjecture}{conjectures}
\Crefname{conj}{Conjecture}{Conjectures}
\crefname{obs}{observation}{observations}
\Crefname{obs}{Observation}{Observations}
\crefname{assump}{assumption}{assumptions}
\Crefname{assump}{Assumption}{Assumptions}
\crefname{rem}{remark}{remarks}
\Crefname{rem}{Remark}{Remarks}

\title{Local Adaptivity in Federated Learning: Convergence and Consistency}

\author{Jianyu Wang$^{\dagger}$\thanks{Work performed while doing an internship at Google Research. Emails: {jianyuw1@andrew.cmu.edu, \{xuzheng, zachgarrett, zachcharles, luyangliu\}@google.com, gaurij@andrew.cmu.edu}}, 
Zheng Xu$^\mathsection$, Zachary Garrett$^\mathsection$, Zachary Charles$^\mathsection$,
Luyang Liu$^\mathsection$, Gauri Joshi$^\dagger$ \\ \\
$^\dagger$Carnegie Mellon University, $^\mathsection$Google Research}

\begin{document}

\maketitle

\begin{abstract}
The federated learning (FL) framework trains a machine learning model using decentralized data stored at edge client devices by periodically aggregating locally trained models. Popular optimization algorithms of FL use vanilla (stochastic) gradient descent for both local updates at clients and global updates at the aggregating server. Recently, adaptive optimization methods such as AdaGrad have been studied for server updates. However, the effect of using adaptive optimization methods for local updates at clients is not yet understood. We show in both theory and practice that while local adaptive methods can accelerate convergence, they can cause a non-vanishing solution bias, where the final converged solution may be different from the stationary point of the global objective function. We propose correction techniques to overcome this inconsistency and complement the local adaptive methods for FL. Extensive experiments on realistic federated training tasks show that the proposed algorithms can achieve faster convergence and higher test accuracy than the baselines without local adaptivity.
\end{abstract}

\section{Introduction}
Federated learning (FL) is an emerging paradigm to perform distributed machine learning model training on decentralized edge clients (\eg mobile phones or IoT devices) under the orchestration of a central server, while keeping private training data on the client devices~\citep{kairouz2019advances}. In the cross-device FL setting, a global model is usually trained by a collaborative process of thousands or even millions of participating clients. In each training round of FL, the central server broadcasts the global model to a random subset of clients to perform local model training, and each participated client only uploads the model parameter changes back to the server. Then, the server aggregates local changes to update the global model and continues the next round. This FL training procedure was proposed by~\citet{mcmahan2016communication} as the \emph{federated averaging} (\fedavg) algorithm, and widely applied in diverse applications~\citep{hard2018federated,brisimi2018federated,xu2020federated}.

Recently, \citet{reddi2020adaptive} proposed a generalization of \fedavg referred to as \fedopt. \fedopt is a flexible algorithmic framework that allows the clients and the server to choose different optimization methods (which are referred to as \textsc{ClientOpt} and \textsc{ServerOpt}) more general than stochastic gradient descent (SGD) in \fedavg. The key idea is to treat the aggregated local changes from clients as a ``pseudo-gradient'' and use it as input to \textsc{ServerOpt} when updating the global model. A few previous works explored the choices of server optimizer (\textsc{ServerOpt}) in the \fedopt framework. For example, \citet{hsu2019measuring,Wang2020SlowMo} used SGD with momentum at the server while keeping client optimizer as SGD and observed significant empirical improvements.

Adaptive methods, such as \adagrad~\citep{mcmahan2010adaptive,duchi2011adaptive}, \adam~\citep{kingma2014adam}, have achieved superior empirical performance over SGD in centralized training of machine learning models for some applications. In particular, \adam and its variants~\citep{reddi2019convergence,zaheer2018adaptive,zhuang2020adabelief} are widely recognized as the preferred optimizers for language-related training tasks. Motivated by their superior performance, \citet{reddi2020adaptive} studied a specific class of \fedopt, where \textsc{ClientOpt} is still SGD but \textsc{ServerOpt} uses adaptive methods. It has been validated through extensive experiments in~\citep{reddi2020adaptive,tong2020effective} that the server-only adaptive methods can achieve faster convergence than vanilla \fedavg in many federated training tasks. 
\begin{figure}[t]
    \centering
    \begin{subfigure}{0.45\textwidth} 
    \includegraphics[width=\columnwidth]{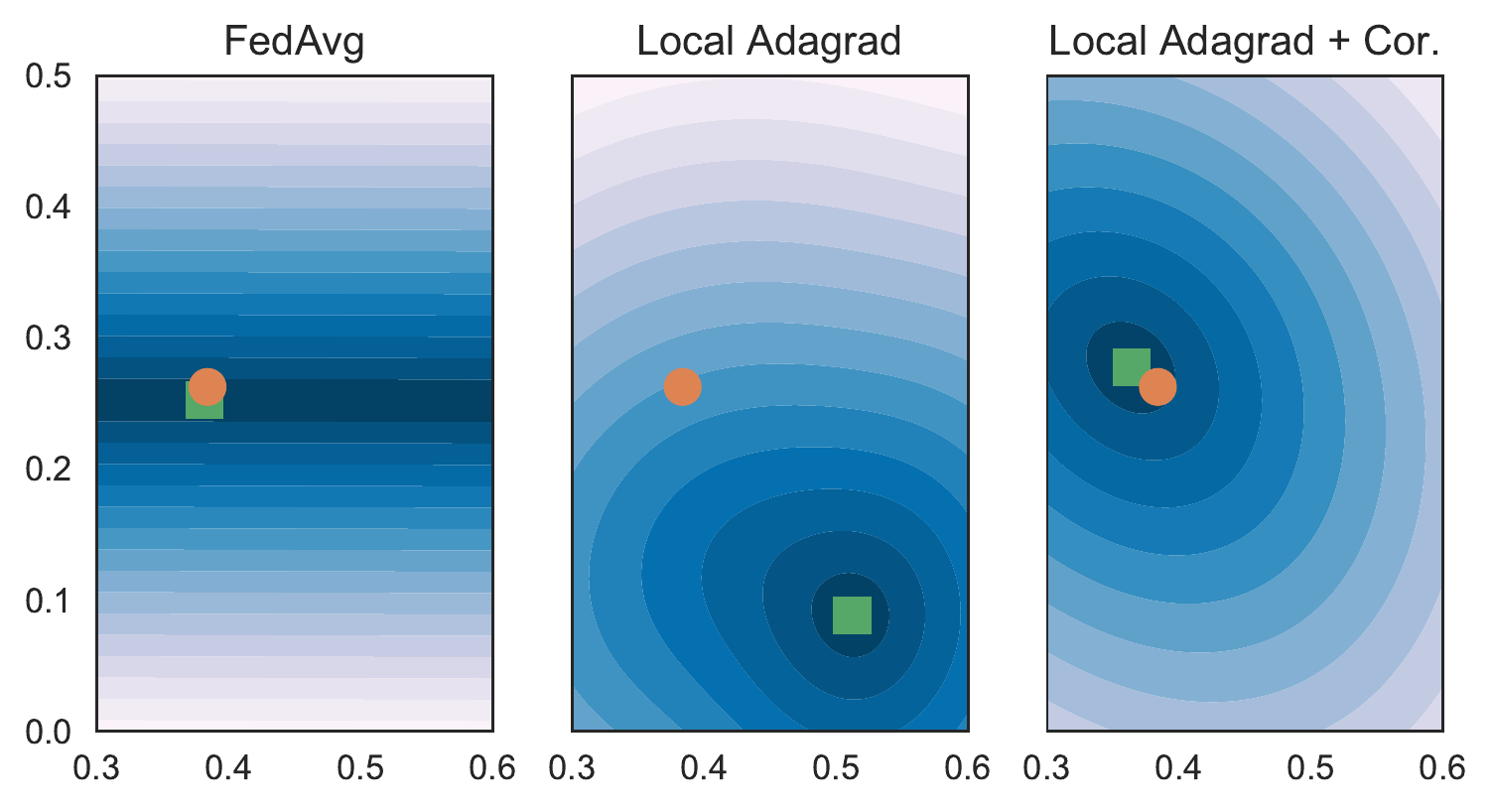}
    \caption{}
    \label{fig:loss_surface}
    \end{subfigure}%
    ~
    \begin{subfigure}{0.4\textwidth} 
    \includegraphics[width=\columnwidth]{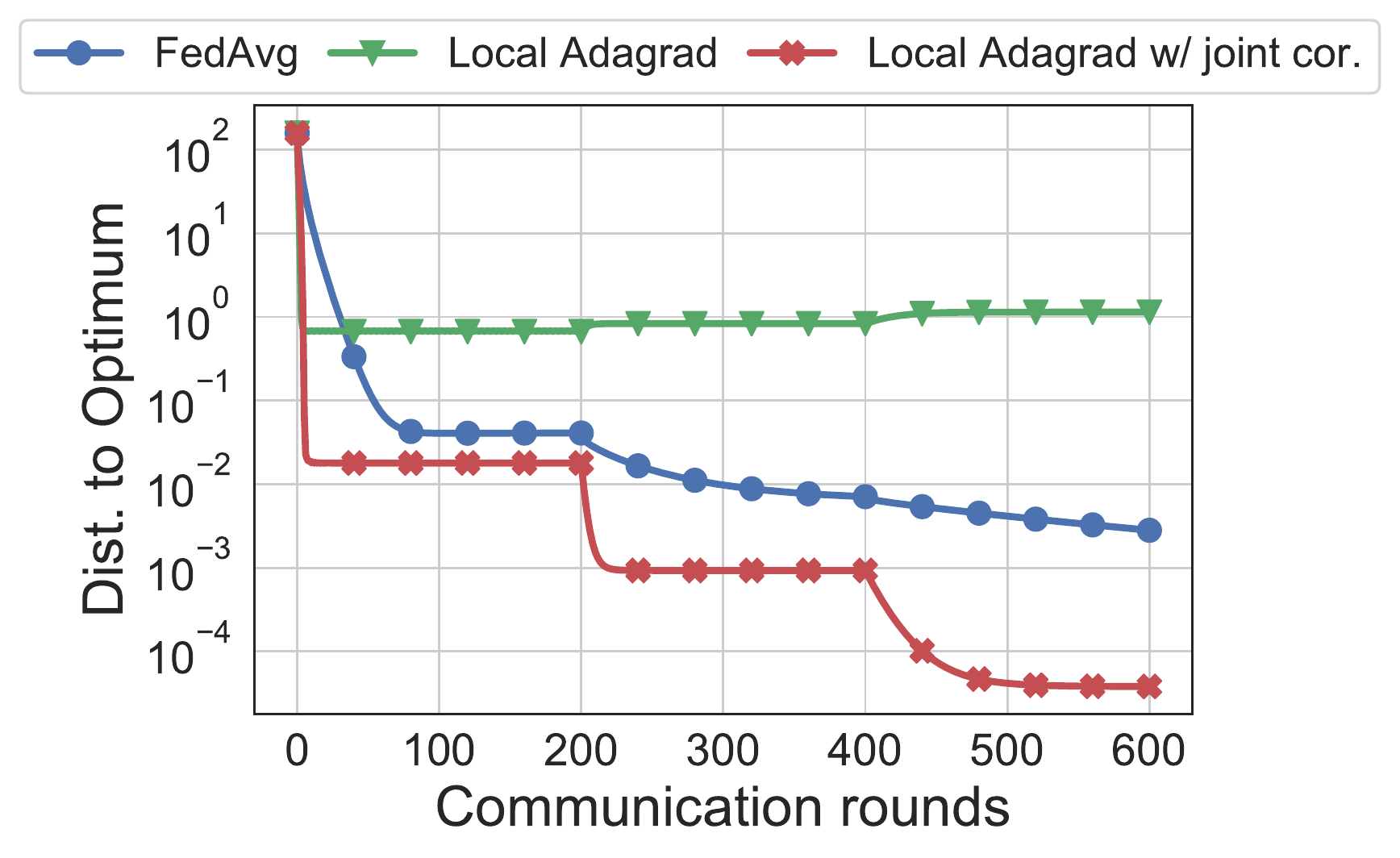}
    \caption{}
    \label{fig:qm_simulation}
    \end{subfigure}
    \caption{The effects of local adaptive methods on a synthetic 2D quadratic problem. Instantiations of \fedopt with local adaptivity can be viewed as finding the fixed points of some operator $\Exs[\mathcal{A}]$. For various methods, we plot (a) the norm $\|\vx-\Exs[\mathcal{A}(\vx)]\|$, where darker colors indicate closer to the fixed point 
    (Under local \adagrad, the optimization landscape is better conditioned, but the \textcolor{ForestGreen}{fixed-point} is inconsistent with the \textcolor{orange}{global minimum} of the original global objective.); (b) training curves where the client learning rate is decayed at rounds $200, 400$. We develop correction methods that can correct the inconsistency, while retaining the better conditioning.
    }
    \label{fig:illustraion}
    \vspace{-1em}
\end{figure}

Despite the success of using server-side adaptive methods, the power of adaptivity has not been fully utilized in the \fedopt framework, particularly on the client side. Instead of using adaptive methods only in \textsc{ServerOpt} for periodical updates per communication round, one can also choose to use adaptive methods in \textsc{ClientOpt} for every local iteration, or in both \textsc{ClientOpt} and \textsc{ServerOpt}. 
Client adaptive methods can accelerate the local convergence by better utilizing the geometric information of the local objectives. Unfortuantely, there are little to no literature exploring this promising direction. It remains an open problem whether the faster local convergence can be translated into faster global convergence and savings in communication rounds. 

In this paper, we explore the usage of local adaptive methods and provide affirmative answers to the question: \emph{Can federated learning effectively use adaptive optimization methods on local clients?}
Specifically, our main contributions are listed as follows.
    
    1. We identify that naively changing \textsc{ClientOpt} in the \fedopt framework from SGD to adaptive methods or other stateful optimizers can be problematic, as it does not define how to update client optimizer states (\ie pre-conditioners and momentum buffers) across rounds. We propose to overcome this discontinuity issue by restarting the update of client optimizer states at the beginning of each round. This simple method enables the usage of local adaptivity in \fedopt. We also show that, on many practical training tasks, using adaptive optimizers on clients  can achieve faster convergence and higher test accuracy than previous methods. \zheng{The previous sentence seems to be a bit vague and the logic on how it connects to the first sentence is not quite clear to me. What is the proposed method, and how did it overcome the issue described in the first sentence? Maybe just merge the two, drop the "identified" claim, and just directly describe we have a method to overcome the challenge.}\JW{Rephrase the complete part.}
    
    2. We further provide a theoretical analysis for strongly convex functions to understand the effects of various client optimizers. Our theorem suggests that performing \fedopt can be viewed as finding the fixed points of some operator $\Exs[\opA]$, the expression of which is determined by client optimizers. Local adaptivity can change $\Exs[\opA]$ in a way that yields faster convergence. However, this enhanced optimization may come with the cost of a \emph{non-vanishing} solution bias -- the point we converge to may be far away from the global minimizer of the original objective function, as illustrated in \Cref{fig:illustraion}.
    
    3. In order to overcome the side effects of using adaptive client optimizers, we propose (1) local correction technique that can help to mitigate the non-vanishing solution bias; (2) global correction technique that can help to perserve the fast convergence property of adaptive methods. Using one or both of them on the top of adaptive client optimizers can achieve the highest test accuracy on all considered federated training tasks.

\textbf{Related Works.}
In terms of using adaptive methods on clients, \citet{xie2019local,karimireddy2020mime} proposed to use the same optimizer states (pre-conditioners and momentum buffers) on all clients, which is similar to server adaptive methods in \cite{reddi2020adaptive} and does not exploit local adaptivity. To the best of our knowledge, this paper is the first to let clients separately update their optimizer states and study the effects of adaptive client optimizers. Besides, there are few recent literature~\cite{wang2020tackling,charles2021convergence,pathak2020fedsplit,malinovskiy2020local} also study the phenomenon that \fedavg-style algorithms can converge to a mismatched solution from the optimal one. We refer readers to \Cref{sec:inconsistency} for detailed comparisons with these works. At last, within the \fedopt framework, there are many more building blocks that can accelerate the convergence or improve the practical performance, for instance, using control variates~\cite{karimireddy2019scaffold,karimireddy2020mime,liang2019variance}, regularizing the local objective functions \cite{li2018federated,li2019feddane,zhang2020fedpd}, sampling important clients more frequently \cite{cho2020client,chen2020optimal,ribero2020communication}. In this paper, we focus on an important but rarely studied topic -- local adaptivity, which is orthogonal and complementary to the above previous works.
\zheng{optional: I am not sure if claiming complementary is good enough or we want to be stronger on the difference, e.g., we care more of a cross-device settings with stateless clients. A reviewer may still ask us to compare with one of these methods. If we make the claim of the first sentence ``There are extensive literature on accelerating the convergence of federated learning algorithms beyond the \fedopt framework'' (which is arguable and assumes strong context of readers to know very nuanced parts of what belongs to \fedopt), we should have a strong support for why we limit ourselves to \fedopt.}\JW{I rewrote the above sentences. PTAL. Start with the local adaptivity. Add one more sentence.}

\section{Preliminaries}\label{sec:related_works}

\textbf{The Federated Learning Setup.}
In federated learning, the goal is to minimize the objective
\begin{align}
    \obj(\vx) = \Exs_{i\sim \clientDist}[ \obj_i(\vx) ]\label{eqn:global_obj_erm}
\end{align}
where $\obj_i(\vx)=\frac{1}{|\data_i|}\sum_{\xi \in \data_i}f_i(\vx,\xi)$ denotes the local objective function at client $i$, and $f_i(\vx,\xi)$ is the loss function where $\xi$ represents one data sample from the local dataset. In contrast to classic distributed learning, federated learning does not allow clients to share their local data with others or with the central server. As a consequence, the local data distributions $\data_i$ may differ across clients. The term $\clientDist$ denotes the distribution over \newadd{the collection of} all clients. The probability of selecting client $i$ is 
$\clientWeight_i=\mathbb{P}_{\clientDist}(i)$, which is proportional to the dataset size of client $i$. Thus, for a finite set of $M$ clients, the objective function becomes $\obj(\vx) = \sum_{i=1}^{M} w_i \obj_i(\vx)$. 


\textbf{Operator View of Client Optimizers.} In this paper, we are going to use operators to represent the general optimizers on clients. In particular, for each client $i$, we have the following definition.
\begin{defn}[Client Optimizer Operator] \label{defn:opA}
We define an operator $\opA_i(\vx;k,\Xi_i,\vs_i): \mathbb{R}^d \rightarrow \mathbb{R}^d$, which outputs the updated model after performing $k$ steps of a given client optimizer on the local objective $F_i$, from initial model $\vx$, with initial optimizer states $\vs_i$ and random sources $\Xi_i$ (\eg noise in stochastic gradients).
\end{defn}
The operator $\opA_i$ defines how the local models and the client optimizer states (\ie pre-conditioners and momentum buffers) are updated. When the parameters $k, \vs_i, \Xi_i$ are given, the output of $\opA_i$ will only depend on the starting point $\vx$. The specific form of $\opA_i$ changes with the choice of optimizers and its analytical expression can be complicated. For example, when the client optimizer is pre-conditioned SGD (a generalized version of \adagrad) \citep{martens2020new}, we have
\begin{align}\label{eqn:adagrad_iter}
\textstyle
    \opA_i(\vx;k,\Xi_i,\vs_i) = \vx - \lr_i\sum_{s=0}^{k-1}\bm{P}_i^{(s)} \sgrad_i(\vx^{(s)};\xi^{(s)})
\end{align}
where $\lr_i$ is the client learning rate, $\vx^{(s)}$ is the local model after $s$ local iterations, $\sgrad(\vx^{(s)};\xi^{(s)})$ denotes the stochastic gradient evaluated on a random mini-batch $\xi^{(s)}$, and $\bm{P}_i^{(s)}$ is referred to as the \emph{pre-conditioner}. In vanilla SGD, the pre-conditioner is just the identity matrix; In \adagrad, the pre-conditioner is updated as follows:
\begin{align}\label{eqn:adagrad_precond}
    \matP_i^{(s)} = ((\matP_i^{(s-1)})^{-2} + \diag\{g(\vx^{(k)};\xi^{(k)})g(\vx^{(k)};\xi^{(k)})\tp\})^{-\frac{1}{2}}
\end{align}
where $\matP_i^{(-1)}=\vs_i$ is the initial optimizer states. In this example, the random sources $\Xi_i=\{\xi^{(s)}\}$ represents all randomness in the stochastic gradients.

\textbf{Operator View of \fedopt.}
Based on \Cref{defn:opA}, at each round $t$ of \fedopt, the server broadcasts the global model $\vx^{(t)}$ as an initial point to clients. Then, client $i$ uses $\opA_i$ to obtain the locally updated model as follows:
\begin{align}\label{eqn:local_update}
    \vx_i^{(t,\tau_i)} = \opA_i(\vx^{(t)};\tau_i,\Xi_i^{(t)},\vs_i^{(t)}).
\end{align}
After that, the server aggregates local models changes $\vx^{(t)} - \vx_i^{(t,\tau_i)}$ and uses \textsc{ServerOpt} to update the global model to $\vx^{(t+1)}$. Since the server optimizer does not influence the properties of $\opA_i$ and the main purpose of this paper is to understand the effects of client optimizers, we assume \textsc{ServerOpt} is gradient descent (GD) unless stated otherwise, which means that each client participates every round and there is no server pre-conditioner. Our results can be naturally extended to other server optimizers and sampled clients without changing the main insights on client optimizers. Then, the GD-based global update rule of \fedopt can be written as
\begin{align}\label{eqn:global_update}
\textstyle
    \vx^{(t+1)} = \vx^{(t)} - \slr \sum_{i=1}^M \clientWeight_i [\vx^{(t)} - \opA_i(\vx^{(t)};\localStep_i,\Xi_i^{(t)},\vs_i^{(t)})]
\end{align}
where $\slr$ denotes the server learning rate for \textsc{ServerOpt}.
\zheng{Maybe worth clarifying the difference compared to the original \fedopt? Did we generalize,  or specify  client optimizer?}
\JW{This is just preliminary. So it is \fedopt itself without any generalizations.} \zheng{I am not sure if $\Xi_i^{(t)},\vs_i^{(t)}$ are part of \fedopt, it is more of "our interpretation" of \fedopt.}\JW{At least, $\Xi_i$ should be in the \fedopt algorithm, as it corresponds to the noise in stochastic gradients.}

\section{How to Apply Adaptive Methods on Clients}\label{sec:simple_restart}
In this section, we describe the challenges of using local adaptivity in \fedopt and propose a simple, yet effective method to enable the usage of local adaptivity in \fedopt without incurring additional communication costs.

\textbf{Challenge in \fedopt: Discontinuity of Client Optimizer States.}
According to the update rules \Cref{eqn:local_update,eqn:global_update} of \fedopt, we observe that \fedopt can not be directly applied to adaptive client optimizers, as it lacks the articulation of handling the initial optimizer states $\vs_i$ across rounds and clients. 
To be specific, in \fedopt, the client updates are isolated not only from each other but also from different rounds due to the local iterations. While \fedopt uses \textsc{ServerOpt} to bridge the gap between local iterates in different rounds (\ie it specifies how to obtain $\vx^{(t+1)}$ from $\{\vx_i^{(t,\localStep_i)}\}$), it does not define how to update the initial optimizer states $\vs_i$ across rounds. We refer to this problem as \emph{discontinuity of client optimizer states}.

A natural but impractical idea to tackle the discontinuity problem is to let each client inherit its own client optimizer states from previous rounds. However, these previous states can be stale and inaccurate because they are evaluated at point $\vx_i^{(t-1, \localStep_i)}$ and do not take account of the possibly dramatic changes from $\vx_i^{(t-1, \localStep_i)}$ to the current iterate $\vx^{(t)}$. When the server only selects a random subset of clients at each round, the staleness will be further exacerbated, as one client may be disconnected from the server for multiple rounds. Besides, this solution is impossible in many applications of FL (\eg the cross-device setting) that requires clients to be stateless and not to maintain any persistent states across rounds, due to privacy and system constraints~\citep{kairouz2019advances}.
Furthermore, one can also choose to synchronize the local optimizer states at the end of each round and use the synchronized states as the initial value for next round. For example, \citet{yu2019linear} has used a similar strategy to synchronize momentum buffer when the clients use momentum SGD. However, in order to aggregate and broadcast the additional optimizer states, this strategy would incur doubled or tripled communication costs compared to \fedopt with vanilla SGD client optimizer. 

\textbf{Our Proposal: Restarting Client Optimizer States.}
To work with the above challenges, we propose a simple yet effective approach. Instead of applying potentially complicated mechanisms to synchronize optimizer states across clients, we restart the updates of client optimizer states at the beginning of each round. That is, resetting the pre-conditioners to a constant and resetting the momentum buffers to zero. This strategy does not require any information from previous rounds, so it is compatible with application settings that require stateless clients. Also, it does not incur addition communication costs, as only the local model changes are aggregated.


To verify the advantage of this simple strategy, in \Cref{fig:qm_simulation}, we provide simulation results on a toy problem of quadratic functions with deterministic gradients. The result shows the restarting strategy on local \adagrad (\ie the green curve) can achieve significantly faster convergence speed, comparing to vanilla \fedavg (\ie the blue curve) when the tolerance to the global optimal is relatively large. The fast convergence is also observed on practical federated training tasks (see \Cref{fig:adagrad_variants} and more discussions in \Cref{sec:exps}). 


\section{Effects of Client Optimizers on Convergence and Consistency}\label{sec:inconsistency}
Another important observation from the toy example in \Cref{fig:qm_simulation} is that local \adagrad converges to a point which is quite far away from the global optimum. Even if we decay the client learning rate, this inconsistency problem still cannot be mitigated or becomes even worse. In this section, we will theoretically analyze the effect of client optimizers on the convergence and provide insights into the trade-off between convergence speed and the consistency of the solution. 

\textbf{\fedopt Performs Fixed Point Iteration.}
When using the restarting strategy proposed in \Cref{sec:simple_restart}, the argument $\vs_i$ representing the initial optimizer state in the operator $\opA_i$ can be omitted since it is set to its default initial value and is the same on all clients. We further define an operator $\opA(\vx) = \sum_{i=1}^M \clientWeight_i \opA_i(\vx;\localStep_i,\Xi_i)$. Accordingly, the global update rule (see Eqn. \Cref{eqn:global_update}) of \fedopt is equivalent to a stochastic fixed point iteration, shown as follows:
\begin{align}\label{eqn:mann_iter}
    \vx^{(t+1)}
    = (1-\slr)\vx^{(t)} +  \slr\Exs[\opA(\vx^{(t)})] +\slr (\opA(\vx^{(t)}) - \Exs[\opA(\vx^{(t)})])
\end{align}
where $\slr$ denotes the server learning rate and the expectation is taken over all random sources in local stochastic gradients at the current round. The deterministic version of \Cref{eqn:mann_iter} is known as Mann's iteration \cite{ryu2016primer} and will converge to the fixed points of $\Exs[\opA]$ under certain conditions. Replacing server GD with another server optimizer can accelerate the convergence but it will not influence the properties and the fixed points of the operator $\Exs[\opA]$. 

We will analyze the convergence of \fedopt (update rule of which is \Cref{eqn:mann_iter}) in the setting where each local objective function $\obj_i(\vx)$ is $L_i$-Lipschitz smooth and $\mu_i$-strongly convex. Besides, we make the following assumptions on the properties of each client optimizer $\opA_i$ and show how the global convergence is influenced by these properties. If the client optimizer is SGD, then the contractive local operator and bounded cumulative variance properties described below follow directly from the strong convexity and smoothness of the objective function, suggesting the following assumptions are reasonable and mild. 

\begin{assump}[Contractive Local Operator]\label{assump:contractive_op}
The local expected operator $\Exs[\opA_i]$ is contractive, satisfying that $\vecnorm{\Exs[\opA_i(\vx;k,\Xi)] - \Exs[\opA_i(\vy;k,\Xi)]}^2 \leq h_i(k) \cdot \vecnorm{\vx - \vy}^2$ for any $\vx, \vy \in \mathbb{R}^d$, where $0 < h_i(k) < 1$ is a decreasing function of the number of local steps $k$.
\end{assump}
\begin{assump}[Bounded Cumulative Variance]\label{assump:bnd_var}
We assume the local stochastic gradient $g_i$ is an unbiased estimator of $\nabla F_i$ and has bounded variance: $\Exs\vecnorm{g_i(\vx) - \nabla F_i(\vx)}^2\leq \sigma^2$. Similarly, the local operator $\opA_i$ has bounded cumulative variance: $\Exs\vecnorm{\opA_i(\vx;k,\Xi) - \Exs[\opA_i(\vx;k,\Xi)]}^2 \leq q_i(k) \cdot \sigma^2$, where $q_i(k) > 0$ is a non-decreasing function of the number of local steps $k$.
\end{assump}
The function $h_i(k)$ in \Cref{assump:contractive_op} measures the local training progress of \textsc{ClientOpt}. In intuition, when \textsc{ClientOpt} has faster convergence rate or takes more local steps, the output of $\Exs[\opA_i]$ will be closer to the minimum $\vx_i^*$, and hence $h_i$ will become smaller. Besides, the function $q_i(k)$ in \Cref{assump:bnd_var} quantifies how the noise in stochastic gradients are accumulated through $k$ local updates. The analytical forms of $h_i, q_i$ depend on the choice of \textsc{ClientOpt}. For example, when the client optimizer is SGD with fixed client learning rate, we formally prove in \Cref{sec:justify_assump} that $h_i(k)=(1-\lr_i\mu_i)^{2k}$ and $q_i(k)=\lr_i^2 k$. For adaptive optimizers, such as \adagrad and \adam, while it is difficult to obtain the analytical form of $h_i$, we empirically validate \Cref{assump:contractive_op} in \Cref{sec:empirical_assump} and found they can yield smaller $h_i$ values than vanilla SGD.

\begin{thm}[Convergence of \fedopt and Minimizer Inconsistency]\label{thm:conv_fedopt}
Under \Cref{assump:contractive_op,assump:bnd_var}, the operator $\Exs[\opA]$ is contractive and has a unique fixed point $\widetilde{\vx}$. After total $T$ communication rounds, the global iterates of \fedopt converge to $\widetilde{\vx}$ at the following rate for some positive constant $c$:
\begin{align}\label{eqn:thm1}
    \Exs\|\vx^{(T)} - \widetilde{\vx}\|^2 
    \leq c \cdot \parenth{\brackets{\sum_{i=1}^M \clientWeight_i h_i}^T\frac{\|\vx^{(0)} - \widetilde{\vx}\|^2}{1-\sum_{i=1}^M \clientWeight_i h_i} + \frac{1}{T}\frac{\sigma^2\sum_{i=1}^\numClients \clientWeight_i^2  q_i}{(1-\sum_{i=1}^M \clientWeight_i h_i)^2}}.
\end{align}
In the presence of data heterogeneity, the fixed point $\widetilde{\vx}$ of $\Exs[\opA]$ is not necessarily the same as the optimum $\vx^*$ of the global objective $\obj$. We refer to this problem as minimizer inconsistency.
\end{thm}
\textbf{Better Client Optimizers Can Accelerate the Global Convergence.}
\Cref{thm:conv_fedopt} provides a nice connection between local and global convergence. Specifically, if clients use better client optimizers that have smaller contraction constant $h_i$, then the worst-case global convergence to the fixed point $\widetilde{\vx}$ can be improved, as the error bound \Cref{eqn:thm1} monotonically increases with $h_i$. Therefore, a natural strategy to speedup the global convergence is to separately minimize $h_i$ on each client while keeping $q_i$ unchanged. One can achieve this by either tuning a proper learning rate for each client, or using adaptive client optimizers which could reduce the effective local condition number. Intuitively, using better client optimizers makes the operator $\Exs[\opA]$ to be more contractive, and hence, it becomes easier to find its fixed points (as corroborated in \Cref{fig:loss_surface}, the loss surface is better conditioned when using local \adagrad). This acceleration effects is complementary to that of the server optimizer, which changes the form of the fixed point iteration without affecting $\Exs[\opA]$. Therefore, one can use both better client and server optimizers to get the best performance, as we will later illustrate in \Cref{sec:exps}.

\textbf{The Minimizer Inconsistency Problem.}
\Cref{thm:conv_fedopt} also shows that the speedup effects of client optimizers may come with a price: the fixed point $\widetilde{\vx}$ can be different from the optimal minimizer $\vx^*$. When the objective function is quadratic, \textsc{ClientOpt} is GD and all clients have the same learning rates $\lr$ and number of local steps $\tau$, it has been shown in \cite{charles2021convergence,pathak2020fedsplit} that $\| \widetilde{\vx} -\vx^* \| =\mathcal{O}(\lr(\tau-1))$ such that one can balance the trade-off between convergence speed and the minimizer inconsistency by gradually decaying the client learning rate. However, through a simple counterexample below, we will show that \emph{using adaptive optimizers or different hyper-parameters on clients leads to an additional gap between $\widetilde{\vx}$ and $\vx^*$, which does not vanish to zero with the learning rate}. This non-vanishing bias explains our empirical observations in \Cref{fig:qm_simulation} and calls for new techniques to overcome it.

\textbf{Quadratic Example for \Cref{thm:conv_fedopt}.} 
In this toy example, we assume the loss function for client $i$ is given by $\obj_i(\vx) = \frac{1}{2}\vx\tp\matH_i\vx - \ve_i\tp\vx+\vc_i$, where $\matH_i$ is symmetric and positive definite and $\ve_i, \vc_i$ are some vectors. Accordingly, the minimum of $F_i$ is simply $\vx^*_i=\matH_i^{-1}\ve_i$ and the minimum of the global objective is $\vx^*=(\sum_{i=1}^M \clientWeight_i\matH_i)^{-1}\sum_{i=1}^M \clientWeight_i \matH_i\vx_i^*$. 
\begin{thm}\label{thm:qm}
For the quadratic problem, if the client optimizer is preconditioned GD with a fixed pre-conditioner $\matP_i$, then $\widetilde{\vx} = (\sum_{i=1}^M \clientWeight_i[\matI-\matK_i])^{-1}\sum_{i=1}^M \clientWeight_i[(\matI - \matK_i)\vx_i^*]\neq \vx^*$, where $\matK_i=(\matI-\lr_i\matP_i\matH_i)^{\localStep_i}$. If we let $\lr_i = \gamma_i \lr$ and $\lr$ approach to zero, then it follows that
\begin{align}
\textstyle
    \lim_{\lr \rightarrow 0 }\widetilde{\vx} 
    = (\sum_{i=1}^M \clientWeight_i\gamma_i\localStep_i\matP_i\matH_i)^{-1}(\sum_{i=1}^M \clientWeight_i\gamma_i\localStep_i\matP_i\matH_i\vx_i^*) \neq \vx^*.
\end{align}
\end{thm}
\Cref{thm:qm} shows that as long as one of the factors (\ie client learning rate, the number of local steps, the pre-conditioners) is different across clients, then there would appear a non-vanishing gap between $\widetilde{\vx}$ and $\vx^*$. In order to get an intuition behind this phenomenon, we can check the exact expression of $\vx^* - \opA(\vx^*)$. If $\widetilde{\vx}=\vx^*$, then $\vx^* - \opA(\vx^*)$ should always be zero. However, note that
\begin{align}\label{eqn:qm_x-Ax}
\textstyle{
    \vx^* - \opA(\vx^*)
    = \sum_{i=1}^M w_i[\matI - (\matI - \lr_i\matP_i \matH_i)^{\tau_i}]\matH_i^{-1}\nabla \obj_i(\vx^*).
}
\end{align}
In Eqn.~\Cref{eqn:qm_x-Ax}, \emph{each local objective's gradient is implicitly weighted by a matrix} $\matM_i=[\matI - (\matI - \lr_i\matP_i \matH_i)^{\tau_i}]\matH_i^{-1}$. The aggregated local changes $\vx^* - \opA(\vx^*)=\sum_{i=1}^M \clientWeight_i \matM_i \nabla\obj_i(\vx^*)$ can be consisdered as a skewed version of the global gradient $\nabla \obj(\vx^*) = \sum_{i=1}^M \clientWeight_i \nabla \obj_i(\vx^*)=0$. Only if $\matM_i=\matM_j$ for all client pairs $i,j$, we can conclude that $\vx^*= \opA(\vx^*)$ and hence $\widetilde{\vx}= \vx^*$. However, due to the data heterogeneity, the condition $\matM_i=\matM_j$ does not hold. Furthermore, if we omit higher order terms $\lr_i^2$ in \Cref{eqn:qm_x-Ax}, then we get the following via Taylor approximation:
\begin{align}\label{eqn:qm_approx}
\textstyle
    [\vx^* - \opA(\vx^*)]/(\max_i \lr_i \tau_i) = \sum_{i=1}^M w_i \frac{\lr_i\tau_i\matP_i}{\max_i \lr_i \tau_i}\nabla\obj_i(\vx^*) + \mathcal{O}(\lr_i\localStep_i)
\end{align}
where the first term (on the right hand side) corresponds to the non-vanishing gap between $\widetilde{\vx}$ and $\vx^*$, and the second term can be omitted when the client learning rate is sufficiently small. It is worth noting that if $\lr_i=0$ or $\tau_i=0$, then $\vx=\opA(\vx)$ for any $\vx \in \mathbb{R}^d$, as there is no optimization progress at all. Therefore, in order to exclude this invalid solution, we divide $(\max_i \lr_i\tau_i)$ in \Cref{eqn:qm_approx}.

This toy example also illustrates how local adaptivity can accelerate convergence. In particular, when the local pre-conditioners are ideal (\ie inverse of local Hessians), \fedopt can converge to $\widetilde{\vx}$ in just one round by choosing proper server and client learning rates. But if the local pre-conditioner are the same across all clients ($\matP_i=\matP$), then the convergence of \fedopt will depend on the condition numbers of $\{\matP\matH_i\}$ and require multiple rounds.

\textbf{Connections with Previous Works on Minimizer Inconsistency.}
When client learning rates, number of local steps are the same across all clients, and non-adaptive, deterministic \textsc{ClientOpt} are used, $\| \widetilde{\vx} -\vx^* \|$ can vanish to zero along with the learning rates. This phenomenon has been observed and analyzed by few recent literature in different forms, see \citep{charles2021convergence,malinovskiy2020local,pathak2020fedsplit}. \Cref{thm:conv_fedopt} generalizes these results by allowing heterogeneous local hyper-parameters and adaptive, stochastic client optimizers. In addition, the non-vanishing bias was studied in \cite{wang2020tackling} by assuming different local learning rates and local steps at clients. In this paper, we further generalize the results by showing that even when the learning rates and local steps are the same, using local adaptive methods will lead to a non-vanishing gap.  We summarize the differences in \Cref{tab:comparison_recent_works} of \Cref{sec:comparison_recent_works}.

\section{Correction Techniques to Overcome the Non-vanishing Solution Bias}\label{sec:corrections}
As we discussed in \Cref{sec:inconsistency}, while using faster client optimizers and exploiting local structures can help improve the global convergence, this strategy results in a non-vanishing gap between the converged point $\widetilde{\vx}$ and the optimal solution $\vx^*$. Inspired by the quadratic example in \Cref{sec:inconsistency}, we propose (i) a local correction method that reweights local gradients before sending the client updates to the server in order to overcome the non-vanishing part of minimizer inconsistency, and (ii) a global correction method applied by the server when aggregating the client updates in order to preserve fast convergence. 

\textbf{Local Correction: Reweighting Local Gradients.}
From \Cref{eqn:qm_approx} in the quadratic example, one can observe that the non-vanishing minimizer inconsistency comes from the fact that \fedopt \emph{implicitly and improperly} weights each local gradients by a matrix $\matN_i=\lr_i\localStep_i\matP_i$. So a natural solution to overcome the non-vanishing bias is to normalize the local model changes before aggregating them on the server. Instead of sending $\vx - \opA_i(\vx)= \matN_i \nabla\obj_i(\vx) + \mathcal{O}(\lr_i^2)$ to the server, each client can send the locally normalized version $\matN_i^{-1}(\vx - \opA_i(\vx))$. This simple change can ensure that the first term in \Cref{eqn:qm_approx} is always zero no matter which value is the client learning rate. As a consequence, there is no non-vanishing solution bias. The local correction technique also automatically avoids the invalid solution $\vx = \opA(\vx), \forall \vx$ by dividing the client learning rate. 

The local correction technique can be extended to more general settings and work with any adaptive optimizers, such as \adagrad, \adam. If the local model changes of a specific client optimizer can be written as 
$\vx - \opA_i(\vx;\tau_i)=\lr_i\sum_{k=0}^{\localStep_i-1}\matB_i^{(k)} \nabla \obj_i(\vx^{(k)})$ where $\matB_i^{(k)}$ is symmetric and positive-definite and $\vx^{(k)}$ denotes the $k$-th iterate during local updates, then we can choose the local correction matrix to be $\matN_i=\lr_i\sum_{k=0}^{\localStep_i-1}\matB_i^{(k)}$ such that
\begin{align}\label{eqn:local_cor_intuition}
\textstyle
    \sum_{i=1}^M \clientWeight_i \matN_i^{-1}(\vx^* - \opA_i(\vx^*))
    = \underbrace{\textstyle{\sum_{i=1}^M} \clientWeight_i \nabla F_i(\vx^*)}_{=\nabla F(\vx^*)=0} + \mathcal{O}(\lr_i\localStep_i) = \mathcal{O}(\lr_i\localStep_i).
\end{align}
If the client optimizer is \adagrad, then $\matB_i^{(k)}$ is just the pre-conditioner $\matP_i^{(k)}$ used at the $k$-th local iteration; if the client optimizer is \adam, then $\matB_i^{(k)}$ is a weighted summation of local pre-conditioners. In \Cref{sec:pseudo-code}, we provide the pseudo-code of finding the expression of $\matN_i$ for common optimizers; and in \Cref{sec:conv_local_cor}, we formally prove that local correction can help \fedopt with deterministic client optimizers converge to the stationary points of the original objective function (even when it is non-convex). For stochastic client optimizers, we empirically validate the effectiveness of local correction through extensive experiments in \Cref{sec:exps}.

\textbf{Global Correction: Preserving the Fast Convergence.} 
Equation \Cref{eqn:local_cor_intuition} shows that using the local correction technique may change the scale of the aggregated model updates. In particular, the first-order approximation of the aggregated model updates just equals to plain gradients (the first term in \Cref{eqn:local_cor_intuition}) and may lose the local pre-conditioning effects. In order to address this problem, one can use either adaptive server optimizer that are more robust to the scale of its inputs, or a novel global correction technique that can help to preserve the scale of local updates. To be specific, in global correction, the server uses $\matN_s^{-1}\sum_{i=1}^M \clientWeight_i \matN_i^{-1}[\vx - \opA_i(\vx)]$ as the pseudo-gradient of \textsc{ServerOpt}, where $\matN_s$ is given as $\matN_s = \sum_{i=1}^M \clientWeight_i \matN_i^{-1}$. There may also exist other better choices of the matrix $\matN_s$. As long as $\matN_s$ does not depend on the client index $i$, it will not influence the fixed point, to which the algorithm converge. We name this technique as global correction, because it is applied on the server side. In order to obtain $\matN_s$, the clients need to send the \newadd{local correction matrices $\matN_i$ to the server}. But the server does not need to broadcast the correction matrix $\matN_s$ back. Therefore, if matrices $\{\matN_i\}$ are diagonal, then the communication cost per round of using global correction is only $1.5\times$ than that of without using it. \zheng{We should probably have one sentence talking about algorithm 1 somewhere.}

\section{Experiments}\label{sec:exps}
\newadd{In this section, we summarize the experimental results and validate the effectiveness of the proposed methods. Specifically, we show that \fedopt with local adaptivity and correction techqniues can achieve faster convergence, higher test accuracy, as well as more robustness against learning rate changes than previous server-only adaptive methods.}

We focus on three language-related tasks~\cite{reddi2020adaptive}, which have the favored sparse structures for adaptive methods: (i) Next word prediciton using RNN on Stack Overflow (SO NWP); (ii) next character prediction using RNN on Shakespeare (Shakes. NCP); (iii) tag prediction using linear regression on Stack Overflow (SO TP). For the first two tasks, we report the validation/test accuracy. For SO TP task, we report the validation/test Recall@5. Moreover, we also evaluate the methods in image classification task on the CIFAR100 dataset~\cite{krizhevsky2009learning}. The detailed descriptions (hyper-parameter tuning ranges and choices) of these federated training tasks are provided in \Cref{sec:exp_details}. 
Our implementation based on the \href{https://github.com/tensorflow/federated}{Tensorflow Federated (TFF)} package will be open-sourced. 
\textbf{Faster Convergence.} 
In \Cref{fig:adagrad_variants}, we first compare the training curves of different ways of using \adagrad in \fedopt on the SO NWP training task. One can either use \adagrad on clients, or on the server, or on both. It can be observed that using local \adagrad can significantly speedup the convergence compared to using vanilla SGD as the client optimizer. In particular, client-only \adagrad is even slightly faster than server-only \adagrad. When the local epoch increases from $1$ to $5$, the improvement of client-only \adagrad over server-only \adagrad becomes more obvious. If we apply \adagrad on both clients and the server, then \fedopt achieves the fastest convergence (about $4\times$ faster than server-only \adagrad to achieve $20\%$ validation accuracy in \Cref{fig:adagrad_variants}a).
\begin{figure}[!t]
    \centering
    \begin{subfigure}{.33\textwidth}
    \centering
    \includegraphics[width=\textwidth]{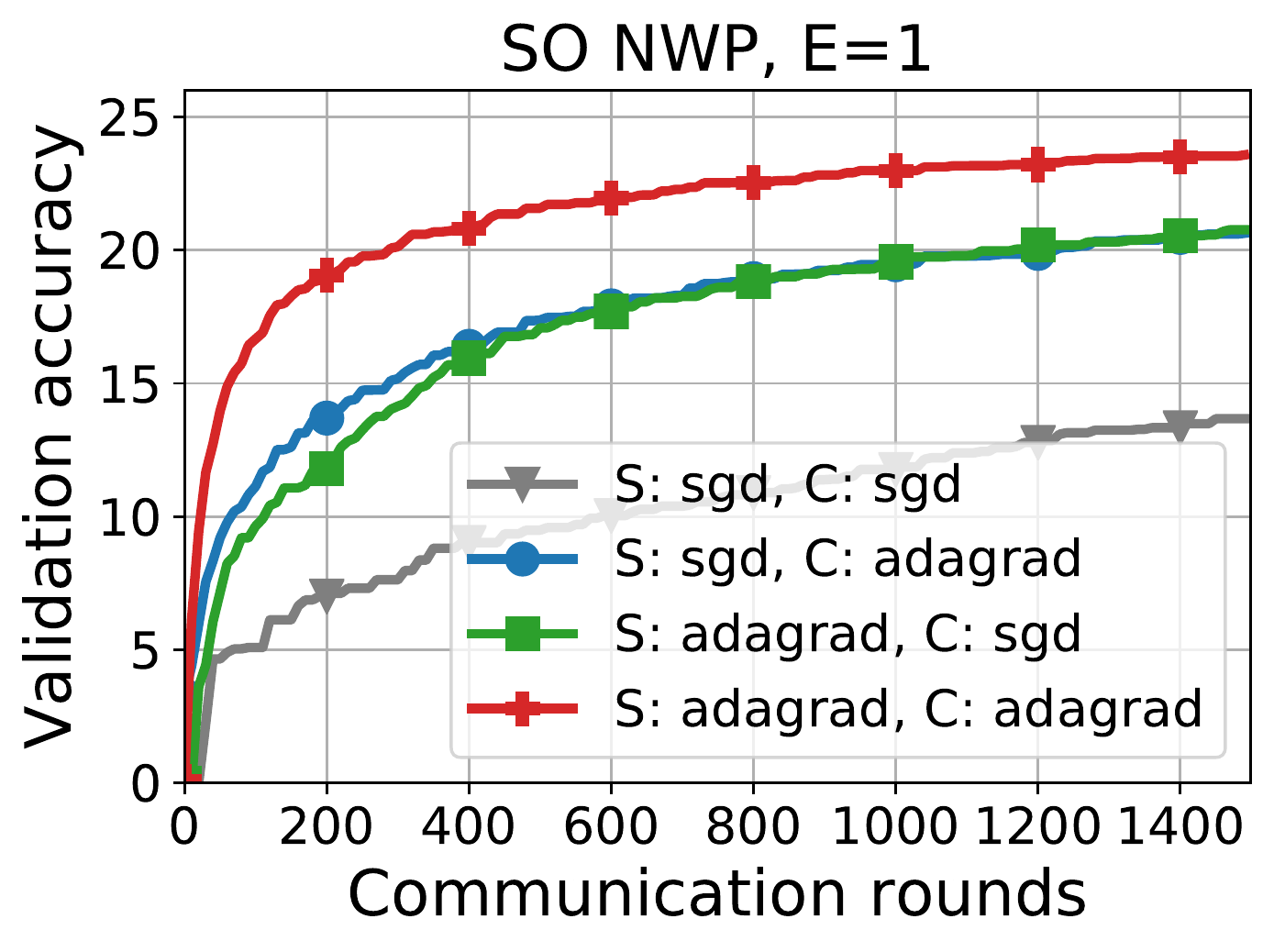}
    \caption{}
    \end{subfigure}%
    ~
    \begin{subfigure}{.33\textwidth}
    \centering
    \includegraphics[width=\textwidth]{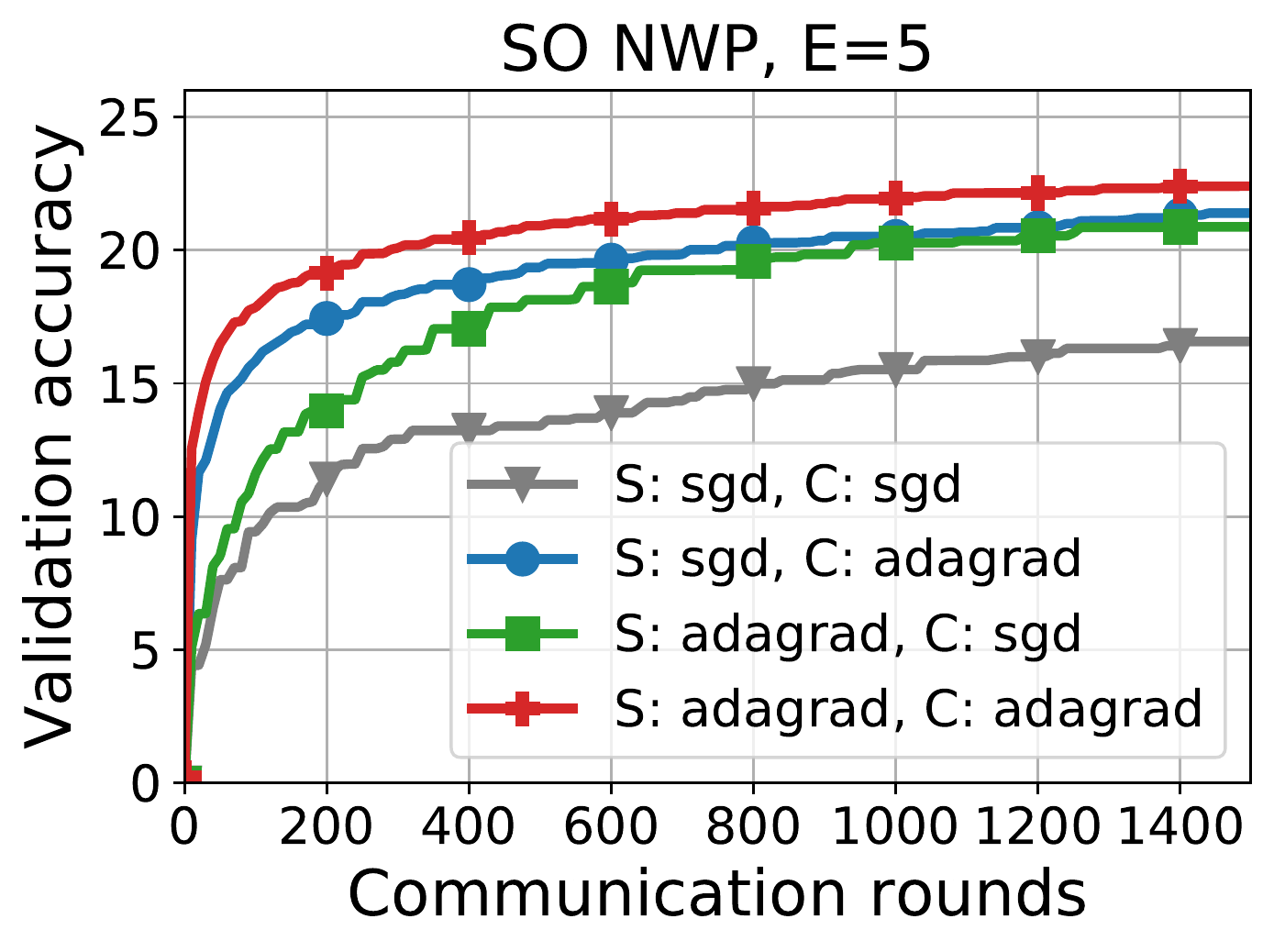}
    \caption{}
    \end{subfigure}%
    \caption{Example of training curves of \fedopt variants on the SO NWP task. The adaptive client optimizer is used with the proposed restarting strategy. 
    In the legend, ``C'' and ``S'' denote the \textsc{ClientOpt} and the \textsc{ServerOpt}, respectively. Using \adagrad on both clients and server is significantly faster than other variants.}
    \label{fig:adagrad_variants}
\end{figure}


\textbf{Higher Final Accuracy.} We further report the test accuracy or Recall@5 on different training tasks in \Cref{tab:final_test_accuracy}. Similar to the above discussions on faster convergence, changing \textsc{ClientOpt} from SGD to \adagrad and using the restarting strategy prposed in \cref{sec:simple_restart} consistently improves the test accuracy on multiple training tasks. And applying \adagrad on both clients and the server achieves the highest test accuracy. For example, when the \textsc{ServerOpt} is \adagrad, the test accuracy improves from $21.80\%$ to $24.40\%$ on SO NWP, and from $57.68\%$ to $57.85\%$ on Shakes. NCP. However, note that on the SO TP task, server \adagrad plus client \adagrad with the restarting strategy performs worse than server-only \adagrad ($65.79$ versus $66.39$). This performance degradation may come from the non-vanishing solution bias, as we discussed in \Cref{sec:inconsistency}.
\begin{table*}[!t]
\small
\centering
    \begin{subtable}[c]{0.333\textwidth}
    \centering
    \begin{tabular}{ rcc }
    \multicolumn{1}{r}{}
     &  \multicolumn{1}{c}{\scriptsize{C:SGD}}
     & \multicolumn{1}{c}{\scriptsize{C:\adagrad}} \\
    \scriptsize{S:SGD} & \cellcolor{orange!2}14.19 & \cellcolor{orange!12}21.68 \\
    \scriptsize{S:\adagrad} & \cellcolor{orange!50}21.80 & \cellcolor{orange!100}24.40 \\
    \end{tabular}
    \caption{SO NWP}
    \end{subtable}%
    \hfill
    \begin{subtable}[c]{0.333\textwidth}
    \centering
    \begin{tabular}{ rcc }
    \multicolumn{1}{r}{}
     &  \multicolumn{1}{c}{\scriptsize{C:SGD}}
     & \multicolumn{1}{c}{\scriptsize{C:\adagrad}} \\
     & \cellcolor{orange!2}57.12 & \cellcolor{orange!12}57.47 \\
     & \cellcolor{orange!50}57.68 & \cellcolor{orange!100}57.85 \\
    \end{tabular}
    \caption{Shakes. NCP}
    \end{subtable}%
    \hfill
    \begin{subtable}[c]{0.333\textwidth}
    \centering
    \begin{tabular}{ rcc }
    \multicolumn{1}{r}{}
     &  \multicolumn{1}{c}{\scriptsize{C:SGD}}
     & \multicolumn{1}{c}{\scriptsize{C:\adagrad}} \\
     & \cellcolor{orange!2}34.50 & \cellcolor{orange!12}39.67 \\
     & \cellcolor{orange!100}66.39 & \cellcolor{orange!50}65.79 \\
    \end{tabular}
    \caption{SO TP}
    \end{subtable}
    \caption{The test accuracy ($\%$) and recall@5 ($\times 100$) after $1500$ communication rounds of \fedopt variants on different training tasks. Darker color means better performance.}
    \label{tab:final_test_accuracy}
\end{table*}
\begin{table*}[t]
\small
\centering
    \begin{tabular}{c | c | c  c c c}                   \toprule
    \multirow{2}{*}{\textbf{Training Tasks}} & \multirow{2}{*}{\textbf{\textsc{ServerOpt}}} & \multicolumn{4}{c}{\textbf{\textsc{ClientOpt}}} \\
    & & SGD  & \adagrad & + Local Cor. & + Joint Cor.\\ \midrule
    SO NWP & \adam & $24.40$ & $24.70$ & $24.81$ & $\mathbf{24.85}$ \\ 
    Shakes. NCP & \adagrad & $57.68$ & $57.85$ & $57.75$ & $\mathbf{58.06}$  \\                     
    SO TP & \adagrad & $66.39$ & $65.79$ & $\mathbf{67.04}$ & $66.94$  \\
    \bottomrule
    \end{tabular}
    \caption{Comparison of our proposed methods (column 2,3,4) and the best server-only adaptive methods (column 1) for each training task. The table shows the test accuracy ($\%$) or recall@5 ($\times 100$) after $1500$ communication rounds. Bolded ones are the best results for each training tasks. In the table, we fix $\epsilon=10^{-7}$ in \textsc{ClientOpt} and tune $\epsilon_s$ in \textsc{ServerOpt}. Therefore, the performance of our proposed methods (column 2,3,4) can be further improved by tuning the $\epsilon$ parameter.}
    \label{tab:plus_correction}
\end{table*}

Then, we evaluate the proposed correction techniques in \Cref{tab:plus_correction} and compare the results with the best server-only adaptive methods on each training tasks (the best server optimizer is selected based on the results in \cite{reddi2020adaptive}). On SO TP, by overcoming the non-vanishing bias, client \adagrad with local correction can achieve much higher Recall@5 than vanilla SGD optimizer ($67.04$ versus $66.39$). Similar improvements also appear in SO NWP and Shakes. NCP tasks. Besides client \adagrad, our proposed methods also work for other common momentum-based adaptive optimizers (such as \adam and \textsc{Yogi}~\cite{zaheer2018adaptive}). We present the experimental results on SO NWP in \Cref{tab:client_adam_yogi}. It can be observed that our proposed method relatively improves the test accuracy of server \adam plus client SGD by $3.9\%$ ($25.35\%$ versus $24.40\%$), and vanilla \fedavg by $78.6\%$ ($25.35\%$ versus $14.19\%$).

\begin{table*}[!t]
\small
\centering
    \begin{tabular}{c c c c}                   \toprule
    \textbf{\textsc{ClientOpt}} &  \textbf{No Cor.} & \textbf{Local Cor.} & \textbf{Joint Cor.} \\ \midrule
    \textsc{Yogi}~\cite{zaheer2018adaptive} & $24.80$ & $25.29$ & $\textbf{25.33}$ \\
    \adam & $24.86$ & $25.15$ & $\textbf{25.35}$ \\
    \bottomrule
    \end{tabular}
    \caption{Performance (test accuracy ($\%$) after $1500$ rounds of training) of momentum-based adaptive client optimizers on SO NWP. The \textsc{ServerOpt} is fixed as \adam. Recall that the test accuracy of sever \adam plus client SGD is $24.40\%$.}
    \label{tab:client_adam_yogi}
\end{table*}

\textbf{Less Sensitive to Hyper-parameter Changes.}
In \Cref{fig:adagrad_heatmap}, we report how the test accuracy changes with server and client learning rates on the SO NWP task. We observe that for server-only \adagrad, there is only two out of $28$ learning rate combinations that can achieve a $20\%+$ test accuracy.
On the other hand, using \adagrad on both clients and server are more robust to the learning rate changes. There are $10$ out of $28$ combinations reaching a $20\%+$ test accuracy. The less sensitivity to learning rate changes can help people to save hyper-parameter tuning time in practice.
\begin{figure}[!t]
    \small
    \centering
    \begin{subfigure}{.33\textwidth}
    \centering
    \includegraphics[width=\textwidth]{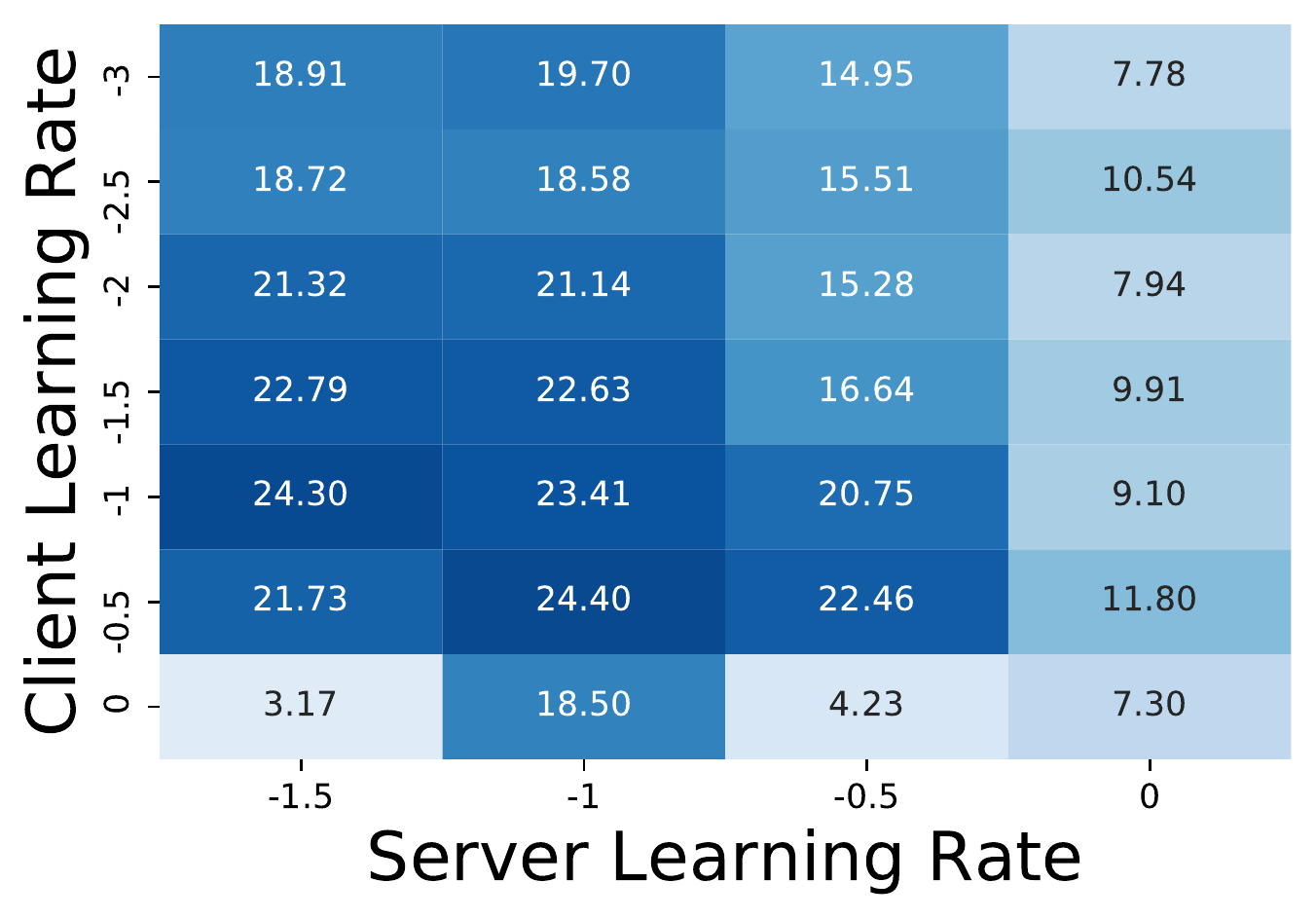}
    \caption{S: \adagrad, C: \adagrad}
    \end{subfigure}%
    ~
    \begin{subfigure}{.33\textwidth}
    \centering
    \includegraphics[width=\textwidth]{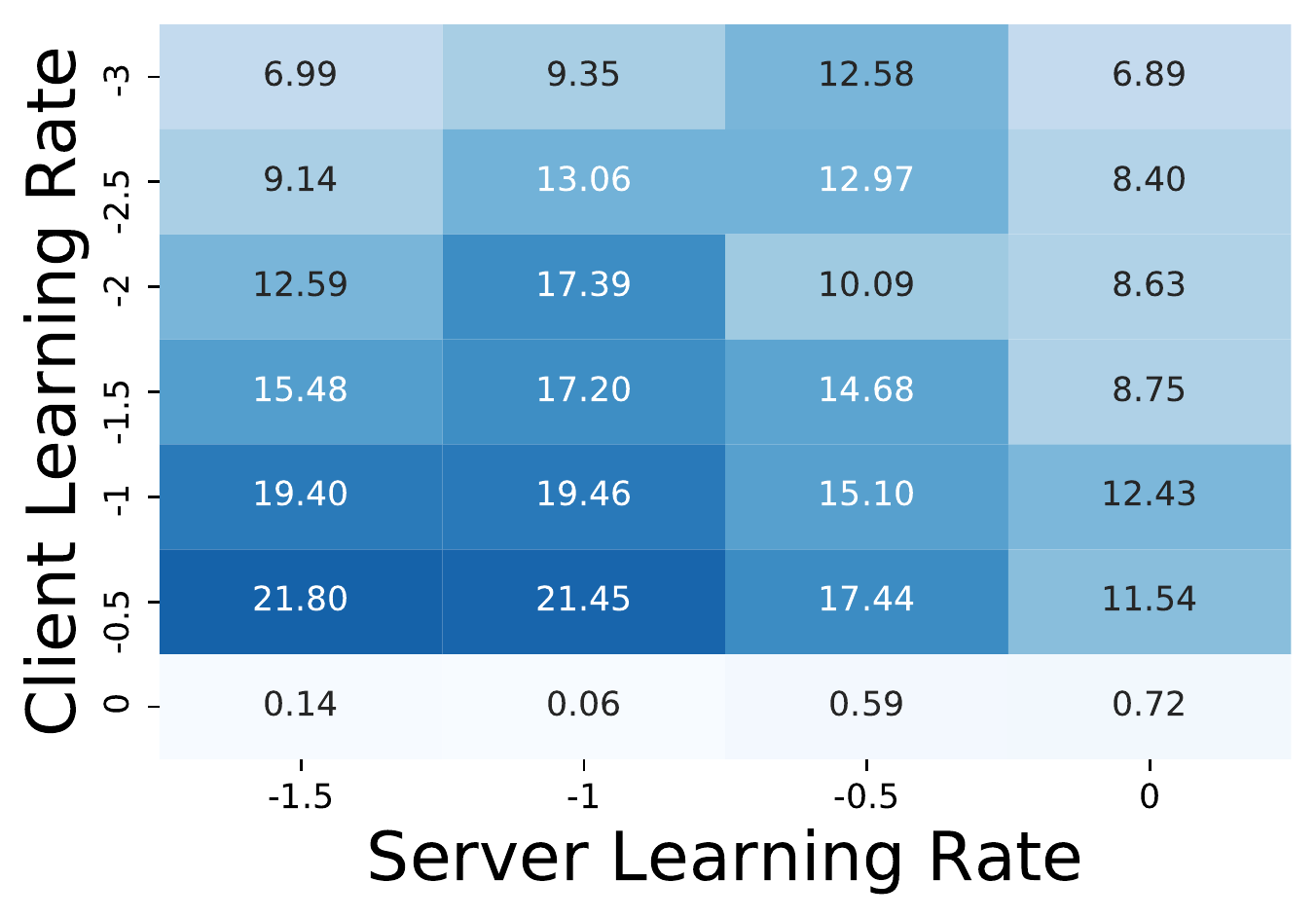}
    \caption{S: \adagrad, C: SGD}
    \end{subfigure}%
    \caption{How the test accuracy changes with server and client learning rate in \fedopt. Using local adaptive methods make the algorithm be more robust to the hyper-parameter changes.}
    \label{fig:adagrad_heatmap}
    \vspace{-2em}
\end{figure}

\textbf{Comparison with the Synchronizing States Strategy.} We further compare the proposed restarting optimizer states strategy to a synchronizing strategy where local pre-conditioners are synchronized after each round. Surprisingly, the additional communication in the synchronized strategy does not help and achieves even worse accuracy. For instance, on Shakes. NCP, server \adagrad plus \newadd{synchronized} client \adagrad has a test accuracy of $56.89\%$, which is lower than that of the restarting one ($57.85\%$).

\newadd{\textbf{Results on Image Classification Task.}
At last, we examine the performance of the proposed methods on pathological version of CIFAR100, following the setup in \cite{reddi2020adaptive}. The observations are similar to language related tasks. Naively using client adaptive methods in \fedopt is not guaranteed to have better performance. But when combining local adaptvity with the proposed correction techniques, the performance consistently improves. Specifically, after $8000$ communication rounds, while server \adam plus client \adagrad has worse test accuracy than server-only \adam ($54.68\%$ vs $56.64\%$), using local correction can improve the accuracy to $56.80\%$.}

\section{Conclusions}
In this paper, we first propose techniques that enable the use of adaptive optimization methods for local updates at clients in federated learning. Through the analysis on effects of client optimizers for smooth and strongly convex functions, we show that although local adaptive methods can accelerate the convergence in some scenarios, it introduces an additional non-vanishing gap between the converged point and the optimal solution. To mitigate the side effects of using local adaptivity, we further propose local and global correction techniques. We verified the advantages of the proposed methods regarding fast convergence, better test accuracy with extensive experiments on benchmark datasets. 

\section*{Acknowledgements}
The authors would like to thank Galen Andrew and Brendan McMahan for helpful discussions.
This research was generously supported in part by NSF grants CCF-1850029 and CCF-2045694, and a Google Computing Platform (GCP) Credit grant.

\bibliographystyle{unsrtnat}
{\bibliography{references.bib}}

\newpage
\appendix
\newpage
\onecolumn
\section{Experimental Details}\label{sec:exp_details}
In this section, we describe our experimental setup in more detail. A brief summary of tunning ranges and default values for hyper-parameters are summarized in \Cref{tab:exp_setting}.
\begin{table*}[!htp]
    \centering
    \begin{tabular}{r c c c}\toprule
        \textbf{Tasks} & SO NWP & Shakes. NCP & SO TP\\
        \textbf{Models} & RNN & RNN & LR \\
        \textbf{Total number of clients} & $342477$ & $715$ & $342477$\\
        \textbf{Maximal elements per client} & $128$ & N/A & $1000$\\\midrule
        \textbf{Active clients per round} & $50$ & $10$ & $10$\\
        \textbf{Mini-batch size} & $16$ & $4$ & $100$ \\
        \textbf{Default Local epochs} & $1$ & $1$ & $1$\\ 
        \textbf{Total rounds} & $1500$ & $1500$ & $1500$\\ \midrule
        \textbf{Client learning rate} $\log_{10}(\lr)$ & \multicolumn{2}{l}{$\{-3, -2.5, \dots, -1, 0, 0.5\}$} & $\{-1,-0.5,\dots,2,2.5\}$ \\
        \textbf{Server learning rate} $\log_{10}(\slr)$ & \multicolumn{2}{l}{$\{-1.5, -1, -0.5, 0, 0.5, 1\}$} & $\{-0.5, 0, 0.5, 1, 1.5\}$\\
        \textbf{Parameter $\epsilon_s$ in \textsc{ServerOpt}} & \multicolumn{3}{c}{$\log_{10}(\epsilon_s)=\{-7, -5, -3, -1\}$}\\
        \textbf{Parameter $\epsilon$ in \textsc{ClientOpt}} & \multicolumn{3}{c}{$\log_{10}(\epsilon)=\{-7\}$}\\\bottomrule
    \end{tabular}
    \caption{Experimental settings, default values and tuning ranges of hyper-parameters.}
    \label{tab:exp_setting}
\end{table*}

\subsection{Training Tasks}
\paragraph{Next-Word Prediction on Stack Overflow (SO NWP for short).}
Stack Overflow is a language modeling dataset from the question and answer site, Stack Overflow. The datasets consists of questions and answers from $342,477$ unique users, each of which is treated as a client in our experiments. We perform next-word prediction task on this dataset and restrict to the 10,000 most frequently used words. The preprocess procedure of this datasets follows \cite{reddi2020adaptive}. Specifically, we let each client only use the first 128 sentences of its local dataset, in order to avoid that clients have extremely different amount of data. Padding and truncation are used to ensure that sentences have 20 words. The metrics we report are the top-1 validation and test accuracy over the $10,000$-word vocabulary. It does not include padding, out-of-vocab, or beginning or end or sentence tokens. The neural network model we use is a RNN with single LSTM layer, which is the same as that of \cite{reddi2020adaptive}.

\paragraph{Tag Prediction on Stack Overfolow (SO TP for short).}
Tag prediction via logistic regression is another training task on the Stack Overflow datasets. It has the same number of clients and vocabulary size as SO NWP. Besides, following the setup in \citep{reddi2020adaptive}, we use the $500$ most frequent tags and a one-versus-rest classification strategy.

\paragraph{Next Character Prediction on Shakespeare (Shakes. NCP for short).}
Shakespeare is another language modeling dataset built from the works from William Shakespeare. Each client corresponds to a speaking role with at least two lines. The lines of each speaking role is splitted into sequences of 80 characters, padding if necessary. The vocabulary size is 90. The neural networks model we use is the same as \cite{reddi2020adaptive}: a RNN with two LSTM layers.

\subsection{Best Performing Hyper-parameters}
We report the best performed hyper-parameters in \Cref{tab:best_hyper_params}. The hyper-parameters are selected in a way such that the average validation accuracy (or recall@5) over the last $100$ rounds achieves the highest value.
\begin{table*}[th]
\small
\centering
    \begin{tabular}{c | c | c c c c}                   \toprule
    \multirow{2}{*}{\textbf{Training Tasks}} & \multirow{2}{*}{\textbf{\textsc{ServerOpt}}} & \multicolumn{4}{c}{\textbf{\textsc{ClientOpt}}} \\\cmidrule(lr){3-6}
    & & SGD  & \adagrad & + Local Cor. & + Joint Cor.\\ \midrule 
    \multirow{3}{*}{SO NWP} & SGD & $(-0.5, 0, \text{NA})$ & $(-0.5,0,\text{NA})$ & $(0,0.5,\text{NA})$ & $(-0.5,0,\text{NA})$ \\
                            & \adagrad & $(-0.5,-1.5,-3)$ & $(-0.5,-1,-5)$ & $(-0.5,-1,-5)$ & $(-0.5,-1,-7)$  \\
                            & \adam & $(-0.5,-1.5,-5)$ & $(-0.5,-1.5,-3)$ & $(-1,-1.5,-5)$ & $(-1,-1.5,-5)$ \\ \midrule
    \multirow{2}{*}{Shakes. NCP} & SGD & $(0,0,\text{NA})$ & $(-0.5,-0.5,\text{NA})$ & $(-0.5,2,\text{NA})$ & $(-0.5,0,\text{NA})$   \\
                            & \adagrad & $(0.5,-1,-1)$ & $(0,-1,-3)$ & $(-0.5,-1,-5)$ & $(-0.5,-0.5,-1)$  \\\midrule                        
    \multirow{3}{*}{SO TP} & SGD & $(2.5,0,\text{NA})$ & $(1.5,0,\text{NA})$ & $(1.5,1.5,\text{NA})$ & $(1.5,0,\text{NA})$   \\
                            & \adagrad & $(0.5,1,-5)$ & $(-0.5,1,-5)$ & $(-0.5,1,-7)$ & $(-0.5,1,-5)$  \\
                            & \adam & $(1.5,-0.5,-5)$ & $(1,-0.5,-5)$ & $(0.5,-0.5,-7)$ & $(1,-0.5,-5)$  \\
    \bottomrule
    \end{tabular}
    \caption{Best performed hyper-parameters in all training tasks. The tuple in each cell corresponds to $(\lr, \slr, \epsilon_s)$: client learning rate, server learning rate, and adaptivity parameter $\epsilon_s$ in the \textsc{ServerOpt}.}
    \label{tab:best_hyper_params}
\end{table*}



\subsection{Pseudo-Codes of The Proposed Algorithms}\label{sec:pseudo-code}
We present the pseudo-codes for our proposed algorithms in \Cref{algo:local_adaptivity}. It is worth noting that \Cref{algo:local_adaptivity} only provides a concrete specification of the proposed local and global correction techniques. Beyond our choices, there may exist many other algorithmic variants.
\begin{algorithm}[!ht]
\small
    \DontPrintSemicolon
    \SetKwInput{Input}{Input}
    \SetAlgoLined
    \LinesNumbered
    \begin{algorithmic}[1]
    \STATE {\bfseries Input:} Initial model $\vx^{(0)}$, \textsc{ClientOpt}, \textsc{ServerOpt}
     \FOR{$t \in \{0,1,\dots,T-1\}$ }
      \FOR{{\bfseries client $i$ in the random set $\activeClients^{(t)}$ in parallel}}
        \STATE Initialization $\vx_i^{(t,0)}=\vx^{(t)}$, \colorbox{blue!25}{$\bm{m}_i^{(t,-1)} = 0, \matP_i^{(t,-1)}=c$}, \colorbox{yellow!25}{$ \bm{N}_i^{(t)} = 0, \matM_i^{(t,-1)}=0$}
        \FOR {$k \in \{0,\dots,\localStep_i-1\}$}
            \STATE Compute local stochastic gradient $\sgrad_i(\vx_i^{(t,k)})$ and update local pre-conditioner $\matP_i^{(t,k)}$
            \STATE Update momentum $\bm{m}_i^{(t,k)} = \beta_1\bm{m}_i^{(t,k-1)} + (1-\beta_1)\sgrad_i(\vx_i^{(t,k)})$
            \STATE Update local model $\vx_i^{(t,k+1)} = \vx_i^{(t,k)} - \lr\matP_i^{(t,k)}\bm{m}_i^{(t,k)}$
            \STATE \colorbox{yellow!25}{$\matM_i^{(t,k)} = \beta_1\matM_i^{(t,k-1)} + (1-\beta_1)\matP_i^{(t,k)}$}
            \STATE \colorbox{yellow!25}{$\bm{N}_i^{(t)} \leftarrow \bm{N}_i^{(t)} + \matM_i^{(t,k)}$}
        \ENDFOR
        \STATE Local changes $\localChange_i^{(t)} =  \vx^{(t)} - \vx_i^{(t,\localStep_i)}$
        \STATE \colorbox{yellow!25}{$\localChange_i^{(t)} \leftarrow (\bm{N}_i^{(t)})^{-1}\localChange_i^{(t)}$}
      \ENDFOR
      \STATE Aggregate $\localChange^{(t)} = \frac{1}{|\activeClients^{(t)}|}\sum_{i \in \activeClients^{(t)}} \localChange_i^{(t)}$
      \STATE \colorbox{green!25}{Aggregate $\matN_s^{(t)} = \frac{1}{|\activeClients^{(t)}|}\sum_{i \in \activeClients^{(t)}}(\bm{N}_i^{(t)})^{-1}$ and set $\Delta^{(t)} \leftarrow (\matN_s^{(t)})^{-1}\Delta^{(t)}$}
      \STATE $\vx^{(t+1)} = \textsc{\textbf{ServerOpt}}(\vx^{(t)}, \localChange^{(t)},\slr,t)$
     \ENDFOR
     \caption{Local adaptive \fedopt with the \colorbox{blue!25}{restarting strategy} and \colorbox{yellow!25}{local correction}, \colorbox{green!25}{global correction} techniques}
     \label{algo:local_adaptivity}
    \end{algorithmic}
\end{algorithm}

\section{Justifications of \Cref{assump:contractive_op,assump:bnd_var} for Vanilla SGD Client Optimizer}\label{sec:justify_assump}
\subsection{Two Useful Lemmas}
Before we dive into the proofs, we would like to first introduce two useful lemmas, which will be repeatedly applied later on.
\begin{lem}
Suppose function $F$ is twice-differentiable. Then, for any points $\vx, \vy \in \mathbb{R}^d$, we have
\begin{align}
    \nabla F(\vx) - \nabla F(\vy) = \matH (\vx - \vy)
\end{align}
where $\matH = \int_0^1 \nabla^2 F(\vy + s(\vx - \vy)) \mathrm{d}s$. 
\end{lem}
\begin{proof} Due to the linearity of integral, we have
\begin{align}
    \nabla F(\vx) - \nabla F(\vy) = \int_0^1 \nabla^2 F(\vy + s(\vx - \vy)) (\vx- \vy)\mathrm{d}s = \matH (\vx - \vy).
\end{align}
\end{proof}

\begin{lem}
Suppose $\vx$ and $\vy$ are random variables with the same random sources $\xi$. Then, if a function $F$ is $L$-lipschitz smooth and $\mu$-strongly convex, then we have
\begin{align}
    \vecnorm{\Exs[\nabla \obj(\vx) - \nabla \obj(\vy)]}^2 \geq& \mu^2\vecnorm{\Exs[\vx - \vy]}^2 \label{eqn:lem2_1}\\
    \vecnorm{\Exs[\nabla \obj(\vx) - \nabla \obj(\vy)]}^2 \leq& L \inprod{\Exs[\vx - \vy]}{\Exs[\nabla \obj(\vx) - \nabla \obj(\vy)]} \label{eqn:lem2_2}\\
    (\mu + L)\inprod{\Exs[\nabla F(\vx) - \nabla F(\vy)]}{\Exs[(\vx - \vy)]}
    \geq& \vecnorm{\Exs[\nabla F(\vx) - \nabla F(\vy)]}^2 + \mu L \vecnorm{\Exs[\vx - \vy]}^2 \label{eqn:lem2_3}
\end{align}
\end{lem}
\begin{proof}
Based on lemma 1, we have $\Exs[\nabla F(\vx) - \nabla F(\vy)] = \Exs[\matH (\vx - \vy)]$. And hence, 
\begin{align}
    \vecnorm{\Exs[\nabla F(\vx) - \nabla F(\vy)]}^2
    =& \inprod{\Exs[\matH (\vx - \vy)]}{\Exs[\matH (\vx - \vy)]} \\
    =& \Exs\inprod{\matH (\vx - \vy)}{\Exs[\matH (\vx - \vy)]}. \label{eqn:lem2_step1}
\end{align}
Since $\mu \preccurlyeq \matH \preccurlyeq L$, it follows that
\begin{align}
    \mu \inprod{\Exs[\vx-\vy]}{\Exs[\matH (\vx - \vy)]} \leq \Exs\inprod{\matH (\vx - \vy)}{\Exs[\matH (\vx - \vy)]} \leq L\inprod{\Exs[\vx-\vy]}{\Exs[\matH (\vx - \vy)]}.\label{eqn:lem2_step2}
\end{align}
Combining \Cref{eqn:lem2_step1,eqn:lem2_step2}, we prove the second inequality in Lemma 2. Furthermore, note that
\begin{align}
    \inprod{\Exs[\vx-\vy]}{\Exs[\matH (\vx - \vy)]}
    = \Exs\inprod{\Exs[\vx-\vy]}{\matH (\vx - \vy)}
    \geq \mu \Exs\inprod{\Exs[\vx-\vy]}{(\vx - \vy)}.
\end{align}
We have
\begin{align}
    \inprod{\Exs[\matH (\vx - \vy)]}{\Exs[\matH (\vx - \vy)]} \geq \mu^2 \inprod{\Exs[\vx-\vy]}{\Exs[\vx-\vy]}.
\end{align}
This completes the proof of the first inequality in Lemma 2. Using the similar technique as above, one can prove that
\begin{align}
    \vecnorm{\Exs[H(\vx - \vy) - \mu (\vx - \vy)]}^2
    \leq& (L - \mu) \inprod{\Exs[\vx- \vy]}{\Exs[(\matH - \mu\matI)(\vx - \vy)]}. \label{eqn:lem2_step3}
\end{align}
Besides, note that
\begin{align}
    \vecnorm{\Exs[H(\vx - \vy) - \mu (\vx - \vy)]}^2
    =& \vecnorm{\Exs[H(\vx- \vy)]}^2 + \mu^2 \vecnorm{\Exs[\vx - \vy]}^2 \nonumber \\
        &- 2\mu \inprod{\Exs[\matH(\vx - \vy)]}{\Exs[\vx - \vy]}. \label{eqn:lem2_step4}
\end{align}
Substituting \Cref{eqn:lem2_step4} into \Cref{eqn:lem2_step3} and rearranging, we obtain that
\begin{align}
    (L+\mu)\inprod{\Exs[\matH(\vx-\vy)]}{\Exs[\vx - \vy]}
    \geq \vecnorm{\Exs[\matH(\vx-\vy)]}^2 + \mu L \vecnorm{\Exs[\vx - \vy]}^2.
\end{align}
Here we complete the proof of the last inequality in Lemma 2.
\end{proof}

\subsection{Proof of Assumption 1 for Vanilla SGD Client Optimizer}
For the ease of writing, we denote $\vx^{(k)}=\opA_i(\vx;k,\Xi)$ and $\vy^{(k)}=\opA_i(\vy;k,\Xi)$ where $\Xi=\{\xi^{(0)},\dots,\xi^{(k-1)}\}$ represents the random sources through $k$ local steps. According to the update rule of SGD, we have
\begin{align}
    \Exs[\vx^{(k+1)}] - \Exs[\vy^{(k+1)}]
    =& \Exs[\vx^{(k)} - \lr g_i(\vx^{(k)})] - \Exs[\vy^{(k)} - \lr g_i(\vy^{(k)})] \\
    =& \Exs[\vx^{(k)} - \vy^{(k)}] - \lr_i \Exs[\nabla F_i(\vx^{(k)}) - \nabla F_i(\vy^{(k)})].
\end{align}
We are going to use \emph{induction} to prove $\vecnorm{\Exs[\vx^{(k)} - \vy^{(k)}]}^2 \leq (1-\lr_i \mu_i)^{2k}\vecnorm{\vx - \vy}^2$. When $k = 1$, we have
\begin{align}
     \vecnorm{\Exs[\vx^{(1)} - \vy^{(1)}]}^2
    =& \vecnorm{\vx - \vy - \lr_i \nabla F_i(\vx) - \nabla F_i(\vy)}^2 \\
    =& \vecnorm{\vx - \vy}^2 + \lr_i^2 \vecnorm{\nabla F_i(\vx) - \nabla F_i(\vy)}^2  -2 \lr_i \inprod{\vx - \vy}{\nabla F_i(\vx) - \nabla F_i(\vy)} \\
    \leq& \parenth{1 - \frac{2\mu_i L_i}{\mu_i + L_i}}\vecnorm{\vx - \vy}^2 + \parenth{\lr_i^2 - \frac{2\lr_i}{\mu_i + L_i}}\vecnorm{\nabla F_i(\vx) - \nabla F_i(\vy)}^2
\end{align}
where the last inequality comes from the Lipschitz smoothness and strongly convexity of $F_i$. When the client learning rate satisfies $\lr_i(\mu_i + L_i) < 2$, we have
\begin{align}
\vecnorm{\Exs[\vx^{(1)} - \vy^{(1)}]}^2
\leq& \parenth{1 - \frac{2\mu_i L_i}{\mu_i + L_i}}\vecnorm{\vx - \vy}^2 + \parenth{\lr_i^2\mu_i^2 - \frac{2\mu_i\lr_i}{\mu_i + L_i}}\vecnorm{\vx - \vy}^2 \\
=& (1-\lr_i\mu_i)^2 \vecnorm{\vx - \vy}^2.
\end{align}
So Assumption 1 is true when $k =1$. Now we assume that Assumption 1 holds for some $k >1$ and examine the value of $\vecnorm{\Exs[\vx^{(k+1)} - \vy^{(k+1)}]}$. In particular, we have
\begin{align}
    \vecnorm{\Exs[\vx^{(k+1)} - \vy^{(k+1)}]}^2
    =& \vecnorm{\Exs[\vx^{(k)} - \vy^{(k)}]}^2 + \lr_i^2 \vecnorm{\Exs[\nabla F_i(\vx^{(k)}) - \nabla F_i(\vy^{(k)})]}^2 \nonumber \\
        &- 2\lr \inprod{\Exs[\vx^{(k)} - \vy^{(k)}]}{\Exs[\nabla F_i(\vx^{(k)}) - \nabla F_i(\vy^{(k)})]}.
\end{align}
According to \Cref{eqn:lem2_1,eqn:lem2_3} of Lemma 2, we have
\begin{align}
    \vecnorm{\Exs[\vx^{(k+1)} - \vy^{(k+1)}]}^2
    \leq& \parenth{1- \frac{2\lr_i \mu_i L_i}{\mu_i L_i}}\vecnorm{\Exs[\vx^{(k)} - \vy^{(k)}]}^2 \nonumber \\
        &+ \parenth{\lr_i^2 - \frac{2\lr_i}{\mu_i + L_i}}\vecnorm{\Exs[\nabla F_i(\vx^{(k)}) - \nabla F_i(\vy^{(k)})]}^2 \\
    \leq& \parenth{1- \frac{2\lr_i \mu_i L_i}{\mu_i L_i}}\vecnorm{\Exs[\vx^{(k)} - \vy^{(k)}]}^2 \nonumber \\
        &+ \parenth{\lr_i^2\mu_i^2 - \frac{2\lr_i\mu_i^2}{\mu_i + L_i}}\vecnorm{\Exs[\vx^{(k)} - \vy^{(k)}]}^2 \\
    \leq& (1-\lr_i\mu_i)^2 \vecnorm{\Exs[\vx^{(k)} - \vy^{(k)}]}^2 \\
    \leq& (1-\lr_i\mu_i)^{2(k+1)}\vecnorm{\vx - \vy}^2.
\end{align}
Here we complete the induction procedure and prove that Assumption 1 holds for vanilla SGD client optimizer and $h_i(k) = (1-\lr_i \mu_i)^{2k}$.

\subsection{Proof of Assumption 2 for Vanilla SGD Client Optimizer}
For the ease of writing, we define $\vx^{(k)} = \opA_i(\vx;k)$ and $\overline{\vx}^{(k)} = \Exs[\vx^{(k)}]$. We are going to use induction to prove that
\begin{align}
    \Exs\vecnorm{\vx^{(k)} - \overline{\vx}^{(k)}}^2 \leq k \lr_i^2 \sigma^2.
\end{align}
When $k = 1$, we have
\begin{align}
    \Exs\vecnorm{\vx^{(1)} - \overline{\vx}^{(1)}}^2
    = \Exs\vecnorm{-\lr_i g_i(\vx;\xi) + \lr_i\nabla F_i(\vx)}^2 
    \leq \lr_i^2 \sigma^2.
\end{align}
We assume that \Cref{assump:bnd_var} holds for some $k > 1$. Then, according to the update rule of SGD, we have
\begin{align}
    &\Exs\vecnorm{\vx^{(k+1)} - \Exs[\vx^{(k+1)}]}^2 \nonumber \\
    =& \Exs\vecnorm{\vx^{(k)} - \lr_i g_i(\vx^{(k)}) - \Exs[\vx^{(k)}] + \lr_i\Exs[\nabla F_i(\vx^{(k)})]}^2 \\
    =& \Exs\vecnorm{- \lr_i g_i(\vx^{(k)}) + \lr_i\nabla F_i(\vx^{(k)}) + \vx^{(k)} - \lr_i\nabla F_i(\vx^{(k)}) - \Exs[\vx^{(k)}] + \lr_i\Exs[\nabla F_i(\vx^{(k)})]}^2 \\
    =& \lr_i^2\Exs\vecnorm{g_i(\vx^{(k)}) - \nabla F_i(\vx^{(k)})}^2 + \Exs\vecnorm{\vx^{(k)} - \lr_i\nabla F_i(\vx^{(k)}) - \Exs[\vx^{(k)}] + \lr_i\Exs[\nabla F_i(\vx^{(k)})]}^2 \\
    \leq& \lr_i^2 \sigma^2 + \Exs\vecnorm{\vx^{(k)} - \lr_i\nabla F_i(\vx^{(k)}) - \Exs[\vx^{(k)}] + \lr_i\Exs[\nabla F_i(\vx^{(k)})]}^2\\
    =& \lr_i^2 \sigma^2 + \Exs\vecnorm{\vx^{(k)} - \Exs[\vx^{(k)}]}^2 + \lr_i^2\Exs\vecnorm{\nabla F_i(\vx^{(k)}) - \Exs[\nabla F_i(\vx^{(k)})]}^2 \nonumber \\
        & -2\lr_i\inprod{\vx^{(k)} - \Exs[\vx^{(k)}]}{\nabla F_i(\vx^{(k)}) - \Exs[\nabla F_i(\vx^{(k)})]} \\
    \leq& (k+1)\lr_i^2\sigma^2+ \lr_i^2\Exs\vecnorm{\nabla F_i(\vx^{(k)}) - \Exs[\nabla F_i(\vx^{(k)})]}^2 \nonumber \\
        & -2\lr_i\Exs\inprod{\vx^{(k)} - \Exs[\vx^{(k)}]}{\nabla F_i(\vx^{(k)}) - \Exs[\nabla F_i(\vx^{(k)})]}. \label{eqn:lem2_step6}
\end{align}
Then, we define $\epsilon = \vx^{(k)} - \overline{\vx}^{(k)}$. Accordingly, we have
\begin{align}
    &\Exs\vecnorm{\nabla F_i(\vx^{(k)}) - \Exs[\nabla F_i(\vx^{(k)})]}^2 \nonumber\\
    =& \Exs_\epsilon\vecnorm{\nabla F_i(\overline{\vx}^{(k)}+\epsilon) - \Exs_\zeta[\nabla F_i(\overline{\vx}^{(k)}+\zeta)]}^2 \\
    =& \Exs_\epsilon\vecnorm{\Exs_\zeta[\nabla F_i(\overline{\vx}^{(k)}+\epsilon) - \nabla F_i(\overline{\vx}^{(k)}+\zeta)]}^2 \\
    \leq& L_i \Exs_\epsilon\inprod{\Exs_\zeta[\epsilon - \zeta]}{\Exs_\zeta[\nabla F_i(\overline{\vx}^{(k)}+\epsilon) - \nabla F_i(\overline{\vx}^{(k)}+\zeta)]} \label{eqn:lem2_step5}\\
    =& L_i \Exs_\epsilon\inprod{\epsilon}{\Exs_\zeta[\nabla F_i(\overline{\vx}^{(k)}+\epsilon) - \nabla F_i(\overline{\vx}^{(k)}+\zeta)]} \\
    =& L_i \Exs\inprod{\vx^{(k)} - \Exs[\vx^{(k)}]}{\nabla F_i(\vx^{(k)}) - \Exs[\nabla F_i(\vx^{(k)})]}
\end{align}
where \Cref{eqn:lem2_step5} is because of the second inequality \Cref{eqn:lem2_2} in Lemma 2. As a consequence, when $\lr_i L_i < 2$, we have
\begin{align}\label{eqn:lem2_step7}
    \lr_i \Exs\vecnorm{\nabla F_i(\vx^{(k)}) - \Exs[\nabla F_i(\vx^{(k)})]}^2 \leq 2\lr_i \Exs\inprod{\vx^{(k)} - \Exs[\vx^{(k)}]}{\nabla F_i(\vx^{(k)}) - \Exs[\nabla F_i(\vx^{(k)})]}.
\end{align}
Substituting \Cref{eqn:lem2_step7} back into \Cref{eqn:lem2_step6}, it follows that
\begin{align}
    \Exs\vecnorm{\vx^{(k+1)} - \Exs[\vx^{(k+1)}]}^2 \leq (k+1)\lr_i^2\sigma^2.
\end{align}
Here we complete the induction and prove that \Cref{assump:bnd_var} holds for SGD and $q_i(k) = k\lr_i^2 \sigma^2$.

\subsection{Empirical Validations for Adaptive Client Optimizers}\label{sec:empirical_assump}
While for vanilla SGD client optimizer, we can get the analytical expressions of $h_i,q_i$ in \Cref{assump:contractive_op,assump:bnd_var}, it can be complicated to perform the same analysis for adaptive client optimizers. So in this subsection, we are going to provide some empirical evidence that adaptive optimizers (such as \adam) also satisfies \Cref{assump:contractive_op} and can yield smaller $h_i$ values than vanilla SGD.

In particular, we evaluate the performance of vanilla SGD and \adam on the MNIST dataset~\cite{deng2012mnist}. For each optimizer, we train two logistic regression models, which start from two different initial points $\vx, \vy$ but traverse the same sequence of mini-batches of data. After repeating the same experiment multiple times with different random seeds, we report $h_i=\|\Exs[\opA_i(\vx;k)] - \Exs[\opA_i(\vy;k)]\|^2/\|\vx-\vy\|^2$ in \Cref{fig:val_assump1}. One can observe that, given a number of local steps $k$, \adam can have a smaller value of $h_i$ than vanilla SGD.
\begin{figure}[!ht]
    \centering
    \includegraphics[width=.5\textwidth]{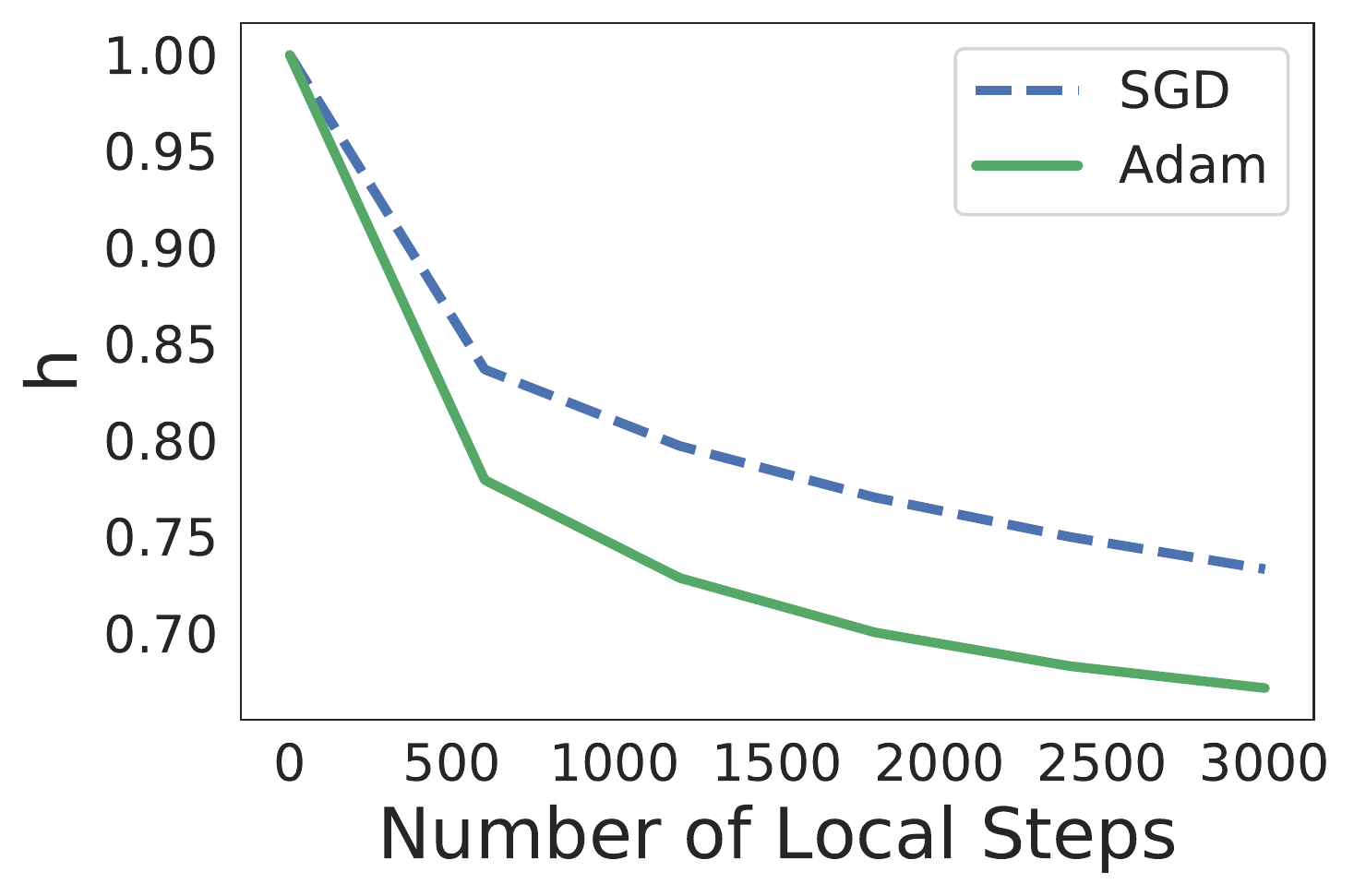}
    \caption{Empirical validation of \Cref{assump:contractive_op} for various client optimizers. The value of $h$ is evaluated by training logistic regression models on the MNIST dataset. For each optimizer, we select the best client learning rate from $\{0.5, 0.2, 0.02, 0.002\}$.}
    \label{fig:val_assump1}
\end{figure}

\section{Proof of Theorem 1}
We first prove $\Exs[\opA]$ is a contractive operator. Based on its definition, we have
\begin{align}
    \vecnorm{\Exs[\opA(\vx)] - \Exs[\opA(\vy)]}^2
    =& \vecnorm{\sum_{i=1}^M \clientWeight_i \Exs[ \opA_i(\vx) - \opA_i(\vy)]}^2 \\
    \leq&  \sum_{i=1}^M \clientWeight_i \vecnorm{\Exs[\opA_i(\vx) - \opA_i(\vy)]}^2 \label{eqn:thm1_step1}\\
    \leq& \sum_{i=1}^M \clientWeight_i h_i \vecnorm{\vx - \vy}^2 \label{eqn:thm1_step2}
\end{align}
where \Cref{eqn:thm1_step1} comes from the Jensen's inequality and \Cref{eqn:thm1_step2} is based on Assumption 1. Since $ 0 \leq \sum_{i=1}^M \clientWeight_i h_i < 1$, according to the Banach fixed-point theorem~\cite{ryu2016primer}, operator $\Exs[\opA]$ is contractive and has a unique fixed point, denoted by $\widetilde{\vx}$.

Then, according to the global update rule \Cref{eqn:mann_iter} of \fedopt, we have
\begin{align}
    \Exs\vecnorm{\vx^{(t+1)} - \widetilde{\vx}}^2
    =& \vecnorm{(1-\slr)(\vx^{(t)} - \widetilde{\vx}) + \slr (\Exs[\opA(\vx^{(t)})] - \widetilde{\vx})}^2 \nonumber \\
     & + \slr^2 \Exs\vecnorm{\sum_{i=1}^\numClients \clientWeight_i\opA_i(\vx^{(t)}) - \Exs[\sum_{i=1}^\numClients \clientWeight_i\opA(\vx^{(t)})]}^2. \label{eqn:thm1_step3}
\end{align}
The first term in \Cref{eqn:thm1_step3} can be bounded as follows
\begin{align}
    &\vecnorm{(1-\slr)(\vx^{(t)} - \widetilde{\vx}) + \slr (\Exs[\opA(\vx^{(t)})] - \widetilde{\vx})}^2 \nonumber \\
    \leq& (1-\slr)\vecnorm{\vx^{(t)} - \widetilde{\vx}}^2 + \slr \vecnorm{\Exs[\opA(\vx^{(t)})] - \widetilde{\vx}}^2 \label{eqn:thm1_step3.5}\\
    \leq& \brackets{1 - \slr (1-\sum_{i=1}^M \clientWeight_i h_i)}\vecnorm{\vx^{(t)} - \widetilde{\vx}}^2 \label{eqn:thm1_step4}
\end{align}
where \Cref{eqn:thm1_step3.5} comes from the fact that for any vectors $\bm{a}, \bm{b}$, we have $\vecnorm{\slr \bm{a} + (1-\slr)\bm{b}}^2 \leq \slr \vecnorm{\bm{a}}^2 + (1-\slr)\vecnorm{\bm{b}}^2$, and the last inequality is from the contraction property of $\Exs[\opA]$. For the second term in \Cref{eqn:thm1_step3}, we have
\begin{align}
    \Exs\vecnorm{\sum_{i=1}^\numClients \clientWeight_i\opA_i(\vx^{(t)}) - \Exs[\sum_{i=1}^\numClients \clientWeight_i\opA(\vx^{(t)})]}^2
    =& \sum_{i=1}^M \clientWeight_i^2 \Exs\vecnorm{\opA_i(\vx^{(t)}) - \Exs[\opA_i(\vx^{(t)})]}^2 \\
    \leq& \sigma^2 \sum_{i=1}^M \clientWeight_i^2 q_i  \label{eqn:thm1_step5}
\end{align}
Substituting \Cref{eqn:thm1_step4,eqn:thm1_step5} back into \Cref{eqn:thm1_step3} we have
\begin{align}
    \Exs\vecnorm{\vx^{(t+1)} - \widetilde{\vx}}^2
    \leq& \brackets{1 - \slr (1-\sum_{i=1}^M \clientWeight_i h_i)}\vecnorm{\vx^{(t)} - \widetilde{\vx}}^2 + \slr^2\sigma^2 \sum_{i=1}^M \clientWeight_i^2 q_i.
\end{align}
Taking the total expectation on both sides, one can get
\begin{align}\label{eqn:thm1_pre_final}
    \Exs\vecnorm{\vx^{(t+1)} - \widetilde{\vx}}^2
    \leq& \brackets{1 - \slr (1-\sum_{i=1}^M \clientWeight_i h_i)}\Exs\vecnorm{\vx^{(t)} - \widetilde{\vx}}^2 + \slr^2\sigma^2 \sum_{i=1}^M \clientWeight_i^2 q_i.
\end{align}

In order to get the final convergence rate, we need a technical lemma from \cite{stich2019unified}, stated as follows.
\begin{lem}[\citet{stich2019unified}]\label{lem:rate}
Suppose there are two non-negative sequences $\{r_t\}, \{s_t\}$ that satisfy the relation
\begin{align}
    r_{t+1} \leq (1-a \gamma_t)r_t - b\gamma_t s_t + c\gamma_t^2
\end{align}
for all $t\geq 0$ and for parameters $b>0, a,c\geq0$ and non-negative stepsizes $\{\gamma_t\}$ with $\gamma_t \leq 1/d$ for a parameter $d \geq a, d >0$. Then, there exists weights $w_t \geq 0, W_T := \sum_{t=0}^T w_t$, such that:
\begin{align}
    \frac{b}{W_T}\sum_{t=0}^T s_t w_t + a r_{T+1} \leq 32 d r_0 \parenth{1-\frac{a}{d}}^{\frac{T}{2}} + \frac{36c}{a T} \leq 32 d r_0 \exp\brackets{-\frac{aT}{2d}} + \frac{36c}{a T}
\end{align}
\end{lem}
By setting $r_t = \Exs\vecnorm{\vx^{(t)} - \widetilde{\vx}}^2, s_t = 0, a = 1-\sum_{i=1}^M \clientWeight_i h_i, c = \sigma^2\sum_{i=1}^M \clientWeight_i^2 q_i, d = 1$, we can obtain from \Cref{lem:rate}:
\begin{align}
    \Exs\vecnorm{\vx^{(T+1)} - \widetilde{\vx}}^2
    \leq& \frac{32 \vecnorm{\vx^{(0)} - \widetilde{\vx}}^2}{1-\sum_{i=1}^M \clientWeight_i h_i} \brackets{\sum_{i=1}^M \clientWeight_i h_i}^{\frac{T}{2}} + \frac{36\sigma^2 \sum_{i=1}^M \clientWeight_i^2 q_i}{T (1-\sum_{i=1}^M \clientWeight_i h_i)^2} \\
    \leq& c\cdot \parenth{\brackets{\sum_{i=1}^M \clientWeight_i h_i}^{\frac{T}{2}}\frac{\vecnorm{\vx^{(0)} - \widetilde{\vx}}^2}{1-\sum_{i=1}^M \clientWeight_i h_i}  + \frac{1}{T}\frac{\sigma^2 \sum_{i=1}^M \clientWeight_i^2 q_i}{(1-\sum_{i=1}^M \clientWeight_i h_i)^2}} \label{eqn:thm1_final}
\end{align}
where $c$ is a positive constant. Here we complete the proof of \Cref{thm:conv_fedopt}.

\paragraph{Special Case: Client Optimizer is GD.}
When the client optimizer is GD (\ie $\sigma = 0$), all clients have the same $\lr, \localStep, \mu, L$, and $\clientWeight_i = 1/M$, we have $h_i = (1-\lr\mu)^{2\tau}$. Then, we can directly set $\slr =1$ in \Cref{eqn:thm1_pre_final}, we get
\begin{align}\label{eqn:thm1_gd}
    \Exs\vecnorm{\vx^{(T)} - \widetilde{\vx}}^2
    \leq& (1-\lr\mu)^{2\tau T}\vecnorm{\vx^{(0)} - \widetilde{\vx}}^2.
\end{align}
If all local operator $\opA_i$ have the same fixed point $\vx_i^*=\vx^*$, then we have $\widetilde{\vx} = \vx^*$ and \Cref{eqn:thm1_gd} recovers the convergence rate of local GD in the IID data setting.

\paragraph{Special Case: Client Optimizer is SGD.}
When the client optimizer is SGD, all clients have the same $\lr, \localStep, \mu, L$, and $\clientWeight_i = 1/M$, we have $h_i = (1-\lr\mu)^{2\tau}$ and $q_i = \localStep\lr^2\sigma^2$. Substituting these into \Cref{eqn:thm1_pre_final}, we get
\begin{align}
     \Exs\vecnorm{\vx^{(T)} - \widetilde{\vx}}^2
    \leq& [1-\slr(1-(1-\lr\mu)^{2\tau})] \cdot \Exs\vecnorm{\vx^{(t)} - \widetilde{\vx}}^2 + \frac{\slr^2 \lr^2 \sigma^2\tau}{M}.
\end{align}
Now we are going to prove that for any $T \geq 0$, with $\slr = 2/[(1-(1-\lr\mu)^{2\tau})(t+\beta)]$,
\begin{align}
    \Exs\vecnorm{\vx^{(T)} - \widetilde{\vx}}^2
    \leq& \frac{4\sigma^2\lr^2 \tau}{M (1-(1-\lr\mu)^{2\tau})^2(T + \beta)} \label{eqn:thm1_step6}
\end{align}
where $\beta$ is a constant, that satisfies that $0 < \beta \leq 4\sigma^2\lr^2\tau/[(1-(1-\lr\mu)^{2\tau})\vecnorm{\vx^{(0)}-\widetilde{\vx}}]^2$. When $t=0$, the inequality \Cref{eqn:thm1_step6} automatically holds according to the definition of $\beta$. Then, we assume \Cref{eqn:thm1_step6} for some $t>1$ and examine the situation $t+1$.
\begin{align}
    \Exs\vecnorm{\vx^{(t+1)} - \widetilde{\vx}}^2
    \leq& \brackets{1 - \frac{2}{t+\beta}}\frac{4\sigma^2\lr^2 \tau}{M(1-(1-\lr\mu)^{2\tau})^2(t + \beta)} + \frac{4\sigma^2\lr^2\tau}{M(1-(1-\lr\mu)^{2\tau})^2(t+\beta)^2} \\
    =& \frac{t+\beta - 1}{t+\beta}\frac{4\sigma^2\lr^2\tau}{M(1-(1-\lr\mu)^{2\tau})^2(t+\beta)} \\
    \leq& \frac{4\sigma^2\lr^2\tau}{M(1-(1-\lr\mu)^{2\tau})^2(t+1+\beta)}.
\end{align}
So \Cref{eqn:thm1_step6} also holds for $t+1$. We complete the induction procedure and conclude that 
\begin{align}
    \Exs\vecnorm{\vx^{(T)} - \widetilde{\vx}}^2
    \leq& \frac{4\sigma^2\lr^2 \tau}{M (1-(1-\lr\mu)^{2\tau})^2(T + \beta)} \\
    =& \frac{\sigma^2}{ \mu^2 M \tau(T+\beta)} \parenth{\frac{2\lr\mu\tau}{1-(1-\lr\mu)^{2\tau}}}^2. \label{eqn:thm1_sgd_final}
\end{align}
When $\tau =1$, we have $\widetilde{\vx} = \vx^*$ and
\begin{align}
    \Exs\vecnorm{\vx^{(T)} - \widetilde{\vx}}^2
    \leq& \frac{\sigma^2}{ \mu^2 M (T+\beta)} \parenth{\frac{2\lr\mu}{1-(1-\lr\mu)^{2}}}^2 \\
    =& \frac{\sigma^2}{ \mu^2 M (T+\beta)} \parenth{\frac{2\lr\mu}{\lr\mu (2-\lr\mu)}}^2 \\
    =& \frac{\sigma^2}{ \mu^2 M (T+\beta)} \parenth{\frac{2}{2-\lr\mu}}^2 \\
    \leq& \frac{4\sigma^2}{ \mu^2 M (T+\beta)} \label{eqn:thm1_sgd_tau1}
\end{align}
where the last inequality follows from $\lr\mu \leq 1$. The result \Cref{eqn:thm1_sgd_tau1} recovers the optimal rate for distributed synchronous SGD~\cite{bottou2016optimization}. When $\tau > 1$, one can obtain that
\begin{align}
    \Exs\vecnorm{\vx^{(T)} - \widetilde{\vx}}^2
    \leq& \frac{\sigma^2}{ \mu^2 M \tau(T+\beta)} \brackets{z(\lr\mu)}^2 \label{eqn:thm1_sgd_final2}
\end{align}
where $z(x) = 2x\tau/(1-(1-x)^{2\tau})$ for $x > 0$. When $\lr\mu \rightarrow 0$, we have $z(\lr\mu) \simeq 1$. In other cases, we are going to prove that if $\lr\mu\tau$ is upper bounded, then $z(\lr\mu)$ can also be upper bounded by some constant. In particular, we first need to prove $z(x)$ is monotonically increasing with $x$ by checking the derivative of $z(x)$:
\begin{align}
    z'(x) 
    &= \frac{2\tau[1-(1-x)^{2\tau}] - 4x\tau^2(1-x)^{2\tau-1}}{[1-(1-x)^{2\tau}]^2} \\
    &= \frac{2\tau}{[1-(1-x)^{2\tau}]^2} \brackets{1 - (1-x)^{2\tau} - 2x\tau(1-x)^{2\tau-1}} \\
    &= \frac{2\tau}{[1-(1-x)^{2\tau}]^2} \brackets{1 - [1+ (2\tau-1)x](1-x)^{2\tau-1}} \\
    &\geq \frac{2\tau}{[1-(1-x)^{2\tau}]^2} \brackets{1 - (1+x)^{(2\tau-1)}(1-x)^{2\tau-1}} \\
    &= \frac{2\tau}{[1-(1-x)^{2\tau}]^2} \brackets{1 - (1-x^2)^{2\tau-1}} > 0.
\end{align}
Suppose $\lr\mu\tau \leq 1$, it follows that
\begin{align}
    \max z(\lr\mu) = z(1/\tau) = \frac{2}{1 - (1-\frac{1}{\tau})^{2\tau}} \leq \frac{2}{1-e^{-2}} < 3. \label{eqn:thm1_zx}
\end{align}
Substituting \Cref{eqn:thm1_zx} into \Cref{eqn:thm1_sgd_final2}, we have
\begin{align}
    \Exs\vecnorm{\vx^{(T)} - \widetilde{\vx}}^2
    \leq& \frac{9\sigma^2}{ \mu^2 M \tau(T+\beta)}
\end{align}
which matches the lower bound of local SGD in the IID data setting~\cite{woodworth2020minibatch}, in which all local operator share the same fixed point $\widetilde{\vx}=\vx^* = \vx_i^*$ for all $i$.

\section{Connection with Previous Works on the Minimizer Inconsistency}\label{sec:comparison_recent_works}
When client learning rates, number of local steps are the same across all clients, and non-adaptive, deterministic \textsc{ClientOpt} are used, $\| \widetilde{\vx} -\vx^* \|$ can vanish to zero along with the learning rates. This phenomenon has been observed and analyzed by few recent literature in different forms, see \citep{charles2021convergence,malinovskiy2020local,pathak2020fedsplit}. \Cref{thm:conv_fedopt} generalizes these results by allowing heterogeneous local hyper-parameters and adaptive, stochastic client optimizers. In addition, the non-vanishing bias was studied in \cite{wang2020tackling} by assuming different local learning rates and local steps at clients. In this paper, we further generalize the results by showing that even when the learning rates and local steps are the same, using local adaptive methods will lead to a non-vanishing gap.  We summarize the differences in \Cref{tab:comparison_recent_works}.
\begin{table}[!ht]
\small
\centering
\begin{tabular}{c | c c c c}    \toprule
Papers & Different $\lr,\localStep$ & Stochastic \textsc{ClientOpt} & Adaptive \textsc{ClientOpt} & Non-vanishing bias \\\midrule
\citep{charles2021convergence} & \xmark & \xmark & \xmark & \xmark \\
\citep{malinovsky2020local} & \xmark & \xmark & \xmark & \xmark   \\
\citep{pathak2020fedsplit} & \xmark & \xmark & \xmark & \xmark   \\ 
\citep{wang2020tackling} & \cmark & \cmark & \xmark & \cmark  \\ 
\rowcolor{gray!20} This paper & \cmark & \cmark & \cmark & \cmark \\
\bottomrule
\end{tabular}
\caption{Comparison with previous works that studied minimizer inconsistency in different forms.}
\label{tab:comparison_recent_works}
\end{table}

\section{Proof of \Cref{thm:qm}}
In the quadratic problem, we can write down the analytical expression of operator $\opA_i$. Specifically, for the $K$-th local iterate of client $i$, we have
\begin{align}
    \vx^{(k+1)} 
    =& \vx^{(k)} - \lr_i \matP_i \nabla F_i(\vx^{(k)}) \\
    =& \vx^{(k)} - \lr_i \matP_i \matH_i(\vx^{(k)} - \vx_i^*) \\
    =& (\matI - \lr_i\matP_i\matH_i)(\vx^{(k)} - \vx_i^*) + \vx_i^*.
\end{align}
That is,
\begin{align}
    \vx^{(k+1)} - \vx_i^* = (\matI - \lr_i\matP_i\matH_i)^{k+1}(\vx - \vx_i^*).
\end{align}
According to the definition of $\opA_i$, we have
\begin{align}
    \opA_i(\vx;\tau_i) 
    =& (\matI - \lr_i\matP_i\matH_i)^{\tau_i}(\vx - \vx_i^*) + \vx_i^*, \\
    \opA(\vx) 
    =& \sum_{i=1}^M \clientWeight_i [(\matI - \lr_i\matP_i\matH_i)^{\tau_i}(\vx - \vx_i^*) + \vx_i^*].
\end{align}
We first show that $\opA$ is contractive. Note that
\begin{align}
    \opA(\vx) - \opA(\vy)
    =& \brackets{\sum_{i=1}^M \clientWeight_i (\matI - \lr_i\matP_i\matH_i)^{\tau_i}} (\vx - \vy).
\end{align}
Therefore, as long as the operator norm of $\sum_{i=1}^M \clientWeight_i (\matI - \lr_i\matP_i\matH_i)^{\tau_i}$ is smaller than $1$, the operator $\opA$ is contractive and has a unique fixed point $\widetilde{\vx}$. Next, we are going to find the analytical expression of $\widetilde{\vx}$. We have
\begin{align}
    \widetilde{\vx} - \opA(\widetilde{\vx})
    =& \sum_{i=1}^M \clientWeight_i [\matI - (\matI - \lr_i\matP_i\matH_i)^{\tau_i} ](\widetilde{\vx} - \vx_i^*) = 0.
\end{align}
After minor rearranging, it follows that
\begin{align}
    \widetilde{\vx} = \brackets{\sum_{i=1}^M \clientWeight_i [\matI - (\matI - \lr_i\matP_i\matH_i)^{\tau_i} ]}^{-1}\brackets{\sum_{i=1}^M \clientWeight_i [\matI - (\matI - \lr_i\matP_i\matH_i)^{\tau_i}] \vx_i^*}.
\end{align}
When $\lr_i = \gamma_i \lr$ and $\lr$ approaches to zero, we have $\matI - (\matI - \lr_i\matP_i\matH_i)^{\tau_i} \simeq \lr_i\tau_i\matP_i\matH_i$ and 
\begin{align}
    \lim_{\lr\rightarrow 0}\widetilde{\vx}
    = \brackets{\sum_{i=1}^M \clientWeight_i \gamma_i\tau_i\matP_i\matH_i}^{-1}\brackets{\sum_{i=1}^M \clientWeight_i \gamma_i\tau_i\matP_i\matH_i \vx_i^*}.
\end{align}
Here we complete the proof.

\section{Proof for the Convergence of Local Correction}\label{sec:conv_local_cor}
\subsection{Main Results}
Without loss of generalities, suppose that at the $t$-th round, the local model changes of client $i$ can be written as
\begin{align}
    \vx^{(t)} - \opA_i(\vx^{(t)};\tau_i) = \lr_i\sum_{k=0}^{\tau_i-1} \matB_i^{(t,k)}\nabla \obj_i(\vx_i^{(t,k)})
\end{align}
where $\lr_i$ is the client learning rate, $\{\matB_i^{(t,k)}\}$ are symmetric and positive definite matrices, and $\vx_i^{(t,k)}$ denotes the local iterate after performing $k$ local steps. When the local correction technique is applied, the client will send the following normalized local changes to the server:
\begin{align}
    \bm{h}_i^{(t)} = \frac{1}{\sum_{k=0}^{\tau_i-1} \matB_i^{(t,k)}}\sum_{k=0}^{\tau_i-1} \matB_i^{(t,k)}\nabla \obj_i(\vx_i^{(t,k)}) := \sum_{k=0}^{\tau_i-1} \matA_i^{(t,k)}\nabla \obj_i(\vx_i^{(t,k)})
\end{align}
where $\matA_i^{(t,k)} = \matB_i^{(t,k)}/\sum_{k=0}^{\tau_i-1} \matB_i^{(t,k)}$ and $\sum_{k=0}^{\tau_i-1} \matA_i^{(t,k)} = \matI$. Then, the server will aggregate the normalized local changes and update the global model as follows
\begin{align}
    \vx^{(t+1)} = \vx^{(t)} - \slr \sum_{i=1}^M \clientWeight_i \bm{h}_i^{(t)} \label{eqn:global_iter_local_cor}
\end{align}
where $\slr$ denotes the server learning rate.

Our convergence analysis will be centered around the following assumptions.
\begin{assump}\label{assump:smooth}
Each local objective is Lipschitz smooth, that is, $\vecnorm{\nabla \obj_i(\vx) - \nabla \obj_i(\bm{y})} \leq L \vecnorm{\vx - \bm{y}}, \forall i \in [M]$.
\end{assump}
\begin{assump}\label{assump:bnd_A}
The matrices $\{\matA_i^{(t,k)}\}$ are positive-definite symmetric matrices and have bounded operator norm: $\opnorm{\matA_i^{(t,k)}}\leq \Lambda/\tau_i$.
\end{assump}
\begin{assump}\label{assump:bnd_grads}
The pre-conditioned gradients at each local iteration have bounded norm, \ie $\vecnorm{\matB_i^{(t,k)}\nabla \obj_i(\vx_i^{(t,k)})}\leq G$.
\end{assump}

\begin{thm}[Convergence Guarantee for Local Correction Technique]\label{thm:convergence}
Suppose all clients have the same client learning rate $\lr$ and the same number of local sptes $\tau$. Under \Cref{assump:smooth,assump:bnd_A,assump:bnd_grads}, if the server learning rate is set as $\slr = \lr\tau \leq 1/L$ and the client learning rate is 
\begin{align}
    \lr = \min \braces{\frac{1}{\tau L}, \frac{1}{\tau T^{\frac{1}{3}}}\parenth{\frac{D}{L^2 \Lambda G^2}}^{\frac{1}{3}}}
\end{align}
where $D = F(\vx^{(0)}) - F_\text{inf}$, then the global iterate \Cref{eqn:global_iter_local_cor} converges at the following rate:
\begin{align}
    \min_{t \in [0, T]} \vecnorm{\nabla \obj(\vx^{(t)})}^2 
    = \mathcal{O}\parenth{\frac{1}{ T} + \frac{\Lambda^{\frac{1}{3}}G^{\frac{2}{3}}}{T^{\frac{2}{3}}}}. \label{eqn:thm3}
\end{align}
\end{thm}
It is worth noting that the convergence rate \Cref{eqn:thm3} matches previous results in \cite{koloskova2020unified,khaled2020tighter} in the deterministic, non-IID data setting. \Cref{thm:convergence} shows that using adaptive client optimizers together with local correction can preserve the same convergence rate as vanilla GD client optimizer and there is no non-vanishing solution bias.
\subsection{Technical Lemmas}
\begin{lem}\label{lem:mat_Jensen}
Suppose $\matA_k \in \mathbb{R}^{d \times d}, k \in [1, K]$ are symmetric positive definite matrices.
\begin{align}
    \vecnorm{\sum_{k=1}^K \matA_k \bm{b}_k}^2 \leq \opnorm{\matA_s}^2 \opnorm{\matA_s^{-1}}\sum_{k=1}^K \opnorm{\matA_k}\vecnorm{\bm{b}_k}^2 \label{eqn:lem_sum_ab}
\end{align}
where $\matA_s = \sum_{k=1}^K \matA_k$.
\end{lem}
\begin{proof}
We define $\widetilde{\matA}_k = (\sum_{k=1}^K \matA_k)^{-1} \matA_k = \matA_s^{-1}\matA_k$. It directly follws that $\sum_{k=1}^K \widetilde{\matA}_k = \matI$. For the left hand side of \Cref{eqn:lem_sum_ab}, we have
\begin{align}
    \vecnorm{\sum_{k=1}^K \matA_k \bm{b}_k}^2
    =& \vecnorm{\matA_s \sum_{k=1}^K \widetilde{\matA}_k \bm{b}_k}^2 = \vecnorm{\matA_s \overline{\bm{b}}}^2 \leq \opnorm{\matA_s}^2 \vecnorm{\overline{\bm{b}}}^2 \label{eqn:lem_lhs}
\end{align}
where $\overline{\bm{b}}=\sum_{k=1}^K \widetilde{\matA}_k \bm{b}_k$. On the other hand, let $\bm{v}_k = \bm{b}_k - \overline{\bm{b}}$ and note that
\begin{align}
    \trace\parenth{\sum_{k=1}^K \widetilde{\matA}_k \bm{b}_k \bm{b}_k\tp}
    =& \trace\parenth{\sum_{k=1}^K \widetilde{\matA}_k (\bm{b}_k - \overline{\bm{b}} + \overline{\bm{b}}) (\bm{b}_k - \overline{\bm{b}} + \overline{\bm{b}})\tp} \\
    =& \trace\parenth{\sum_{k=1}^K \widetilde{\matA}_k \parenth{\bm{v}_k\bm{v}_k\tp + \bm{v}_k \overline{\bm{b}}\tp + \overline{\bm{b}} \bm{v}_k \tp + \overline{\bm{b}} \ \overline{\bm{b}}\tp}} \\
    =& \underbrace{\trace\parenth{\sum_{k=1}^K \widetilde{\matA}_k\bm{v}_k\bm{v}_k\tp}}_{\geq 0} + \underbrace{\trace\parenth{\sum_{k=1}^K \widetilde{\matA}_k\bm{v}_k \overline{\bm{b}}\tp}}_{=0} + \trace\parenth{\sum_{k=1}^K \widetilde{\matA}_k\overline{\bm{b}} \bm{v}_k \tp} + \trace\parenth{\overline{\bm{b}} \ \overline{\bm{b}}\tp}.
\end{align}
For the third term, we have
\begin{align}
    \trace\parenth{\sum_{k=1}^K \widetilde{\matA}_k\overline{\bm{b}} \bm{v}_k \tp} = \trace\parenth{\sum_{k=1}^K \bm{v}_k \tp\widetilde{\matA}_k\overline{\bm{b}}} = \trace\parenth{\sum_{k=1}^K \bm{v}_k \tp\widetilde{\matA}_k\tp\overline{\bm{b}}} = 0.
\end{align}
Therefore, we can obtain that 
\begin{align}
    \vecnorm{\overline{\bm{b}}}^2 = \trace\parenth{\overline{\bm{b}} \ \overline{\bm{b}}\tp} 
    \leq& \trace\parenth{\sum_{k=1}^K \widetilde{\matA}_k \bm{b}_k \bm{b}_k\tp} \\
    =& \sum_{k=1}^K \bm{b}_k\tp\widetilde{\matA}_k \bm{b}_k \\
    \leq& \sum_{k=1}^K \opnorm{\widetilde{\matA}_k}\vecnorm{\bm{b}_k}^2 \\
    \leq& \opnorm{\matA_s^{-1}} \sum_{k=1}^K \opnorm{\matA_k}\vecnorm{\bm{b}_k}^2. \label{eqn:norm_b}
\end{align}
Substituting \Cref{eqn:norm_b} into \Cref{eqn:lem_lhs}, it follows that
\begin{align}
    \vecnorm{\sum_{k=1}^K \matA_k \bm{b}_k}^2 \leq \opnorm{\matA_s}^2 \opnorm{\matA_s^{-1}} \sum_{k=1}^K \opnorm{\matA_k}\vecnorm{\bm{b}_k}^2.
\end{align}

\end{proof}

\subsection{Proof of \Cref{thm:convergence}}
Since each local objective is $\lip$-smooth, we have
\begin{align}
    &\obj(\vx^{(t+1)}) - \obj(\vx^{(t)}) \nonumber \\ 
    \leq& -\alpha \inprod{\nabla \obj(\vx^{(t)})}{\sum_{i=1}^{\numClients}\clientWeight_i\ntg_i^{(t)}} + \frac{\alpha^2 \lip}{2}\vecnorm{\sum_{i=1}^{\numClients}\clientWeight_i\ntg_i^{(t)}}^2 \\ 
    =& -\frac{\alpha}{2}\brackets{\vecnorm{\nabla\obj(\vx^{(t)})}^2 + \vecnorm{\sum_{i=1}^{\numClients}\clientWeight_i\ntg_i^{(t)}}^2 - \vecnorm{\nabla\obj(\vx^{(t)}) - \sum_{i=1}^\numClients \clientWeight_i \ntg_i^{(t)}}^2} + \frac{\alpha^2 \lip}{2}\vecnorm{\sum_{i=1}^{\numClients}\clientWeight_i\ntg_i^{(t)}}^2 \label{eqn:thm2_step1}\\
    \leq& -\frac{\alpha}{2}\vecnorm{\nabla\obj(\vx^{(t)})}^2 + \frac{\alpha}{2}\vecnorm{\nabla\obj(\vx^{(t)}) - \sum_{i=1}^\numClients \clientWeight_i \ntg_i^{(t)}}^2 \label{eqn:thm2_step2}\\
    \leq& -\frac{\alpha}{2}\vecnorm{\nabla\obj(\vx^{(t)})}^2 + \frac{\alpha}{2}\sum_{i=1}^\numClients \clientWeight_i \vecnorm{\nabla\obj_i(\vx^{(t)}) - \ntg_i^{(t)}}^2 \label{eqn:thm2_step3}
\end{align}
where \Cref{eqn:thm2_step1} uses the fact: $\inprod{a}{b} = \frac{1}{2}[\vecnorm{a}^2 + \vecnorm{b}^2 - \vecnorm{a-b}^2]$, \Cref{eqn:thm2_step2} follows from the assumption $\slr \lip < 1$, and \Cref{eqn:thm2_step3} is obtained by applying Jensen's Inequality. For the second term in \Cref{eqn:thm2_step3}, we can further bound it as follows:
\begin{align}
    \vecnorm{\nabla\obj_i(\vx^{(t)}) -\ntg_i^{(t)}}^2 
    =& \vecnorm{\sum_{k=0}^{\localStep_i-1} \matA_i^{(t,k)} \brackets{\nabla\obj_i(\vx^{(t)}) - \nabla \obj_i(\vx_i^{(t,k)})}}^2 \\
    \leq& \sum_{k=0}^{\localStep_i-1}\opnorm{\matA_i^{(t,k)}} \vecnorm{\nabla\obj_i(\vx^{(t)}) - \nabla \obj_i(\vx_i^{(t,k)})}^2 \label{eqn:thm2_step4}\\
    \leq& \lip^2\sum_{k=0}^{\localStep_i-1}\opnorm{\matA_i^{(t,k)}} \vecnorm{\vx^{(t)} - \vx_i^{(t,k)}}^2 \label{eqn:thm2_step5}\\
    \leq& \frac{\lip^2 \Lambda}{\tau_i} \sum_{k=0}^{\localStep_i-1}\vecnorm{\vx^{(t)} - \vx_i^{(t,k)}}^2 \label{eqn:thm2_step6}\\
    =& \frac{\lip^2 \Lambda}{\tau_i}  \sum_{k=0}^{\localStep_i-1}\vecnorm{\Delta_i^{(t,k)}}^2 \label{eqn:thm2_step7}
\end{align}
where \Cref{eqn:thm2_step4} follows \Cref{lem:mat_Jensen}, \Cref{eqn:thm2_step5} is based on the Lipschitz smoothness of the local objectives, and \Cref{eqn:thm2_step6} uses the assumption that matrices $\matA_i^{(t,k)}$ have bounded operator norm. Substituting \Cref{eqn:thm2_step7} into \Cref{eqn:thm2_step3}, we have
\begin{align}
    \obj(\vx^{(t+1)}) - \obj(\vx^{(t)})
    \leq& -\frac{\alpha}{2}\vecnorm{\nabla\obj(\vx^{(t)})}^2 + \frac{\alpha\lip^2\Lambda}{2}\sum_{i=1}^\numClients \frac{\clientWeight_i}{\tau_i}\sum_{k=0}^{\localStep_i-1}\vecnorm{\Delta_i^{(t,k)}}^2.
\end{align}
Taking the sum from $t=0$ to $t=T-1$ and rearranging, we obtain
\begin{align}
    \frac{1}{T}\sum_{t=0}^{T-1}\vecnorm{\nabla\obj(\vx^{(t)})}^2
    \leq& \frac{2(\obj(\vx^{(0)}) - \obj(\vx^{(T)}))}{\slr T} + \frac{\lip^2\Lambda}{T}\sum_{t=0}^{T-1}\sum_{i=1}^\numClients \frac{\clientWeight_i}{\tau_i}\sum_{k=0}^{\localStep_i-1}\vecnorm{\Delta_i^{(t,k)}}^2 \\
    \leq& \frac{2(\obj(\vx^{(0)}) - \obj_{\text{inf}})}{\slr T} + \frac{\lip^2\Lambda}{T}\sum_{t=0}^{T-1}\sum_{i=1}^\numClients \frac{\clientWeight_i}{\tau_i}\sum_{k=0}^{\localStep_i-1}\vecnorm{\Delta_i^{(t,k)}}^2. \label{eqn:thm2_step8}
\end{align}
On the other hand, note that
\begin{align}
    \vecnorm{\Delta_i^{(t,k)}}^2 
    = \lr_i^2\vecnorm{\sum_{s=0}^{k-1}\matB_i^{(t,s)}\nabla \obj_i(\vx_i^{(t,s)})}^2
    \leq k^2\lr_i^2G^2 \leq \tau_i^2 \lr_i^2 G^2. \label{eqn:thm2_step9}
\end{align}
Substituting \Cref{eqn:thm2_step9} into \Cref{eqn:thm2_step8}, we have
\begin{align}
    \frac{1}{T}\sum_{t=0}^{T-1}\vecnorm{\nabla\obj(\vx^{(t)})}^2
    \leq& \frac{2(\obj(\vx^{(0)}) - \obj_{\text{inf}})}{\slr T} + \frac{\lip^2\Lambda}{T}\sum_{t=0}^{T-1}\sum_{i=1}^\numClients \clientWeight_i \lr_i^2\tau_i^2 G^2 \\
    =& \frac{2(\obj(\vx^{(0)}) - \obj_{\text{inf}})}{\slr T} + \lip^2\Lambda G^2 \sum_{i=1}^\numClients \clientWeight_i \lr_i^2\tau_i^2
\end{align}
If we let $\slr=\sum_{i=1}^M \clientWeight_i \lr_i \tau_i \leq \nicefrac{1}{L}$, then it follows that
\begin{align}
    \frac{1}{T}\sum_{t=0}^{T-1}\vecnorm{\nabla\obj(\vx^{(t)})}^2
    \leq& \frac{2(\obj(\vx^{(0)}) - \obj_{\text{inf}})}{\sum_{i=1}^M \clientWeight_i \lr_i \tau_i T} + \lip^2\Lambda G^2 \sum_{i=1}^\numClients \clientWeight_i \lr_i^2\tau_i^2.
\end{align}
When $\lr_i = \lr, \tau_i=\tau, \clientWeight_i = 1/M$, we have
\begin{align}
     \frac{1}{T}\sum_{t=0}^{T-1}\vecnorm{\nabla\obj(\vx^{(t)})}^2
    \leq& \frac{2(\obj(\vx^{(0)}) - \obj_{\text{inf}})}{\lr \tau T} + \lip^2\Lambda G^2\lr^2\tau^2 \\
    \leq& \frac{D L }{ T} + \frac{(D^2 L^2 \Lambda G^2)^{\frac{1}{3}}}{T^{\frac{2}{3}}}
\end{align}
where $D:= \obj(\vx^{(0)}) - \obj_{\text{inf}}$ and the client learning rate is set as
\begin{align}
    \lr = \min \braces{\frac{1}{\tau L}, \frac{1}{\tau T^{\frac{1}{3}}}\parenth{\frac{D}{L^2 \Lambda G^2}}^{\frac{1}{3}}}.
\end{align}

\end{document}